\documentclass[journal]{IEEEtran}
\usepackage{times}

\usepackage[table]{xcolor}
\usepackage[numbers]{natbib}
\usepackage{multicol}
\usepackage[bookmarks=true]{hyperref}

\usepackage{amsmath, amsfonts, amssymb, amsthm}
\usepackage{enumerate}
\usepackage[inline]{enumitem}
\usepackage{mathtools}
\usepackage{graphicx}
\usepackage{svg}
\usepackage{longtable,tabularx}
\usepackage{placeins} 
\usepackage{float}
\usepackage{multirow}
\usepackage{bbm}
\usepackage{threeparttable}
\usepackage{balance}
\usepackage[ruled,algo2e]{algorithm2e}
\usepackage{algpseudocode}
\usepackage{adjustbox}
\usepackage{booktabs}
\usepackage{bm}
\usepackage{etoolbox}
\usepackage{microtype}
\usepackage{caption}
\usepackage{subcaption}
\usepackage{diagbox}
\usepackage{booktabs}
\usepackage{array, makecell}

\usepackage[capitalise]{cleveref}
\usepackage{units}

\newtheorem{theorem}{Theorem}

\newtheorem{proposition}{Proposition}

\newtheorem{corollary}{Corollary}
\newtheorem{remark}{Remark}

\newcommand{\R}{\mathbb{R}}
\newcommand{\mb}[1]{\mathbf{#1}}

\newcommand{\diag}{\mathrm{\mathbf{diag}}}

\DeclareMathOperator*{\argmin}{arg\,min}

\newcommand{\minimize}[1]{ \ensuremath{\underset{#1}{\mathrm{minimize}}\ }}

\newcommand*{\trace}[1]{\ensuremath{\mathrm{trace}\left({#1}\right)}}

\newcommand*{\subj}{\ensuremath{\mathrm{subject\ to\ }}}

\newcommand{\mcal}[1]{\mathcal{#1}}
\newcommand{\mbb}[1]{\mathbb{#1}}

\newcommand{\norm}[1]{\left\Vert #1 \right\Vert}

\newcommand{\SE}{\ensuremath{\mathrm{SE}}}
\newcommand{\SO}{\ensuremath{\mathrm{SO}}}

\newbool{condensed_paper}
\setbool{condensed_paper}{true}

\newbool{further_condense}
\setbool{further_condense}{true}

\newbool{show_slower_image}
\setbool{show_slower_image}{false}

\newbool{enable_concise_mode}
\setbool{enable_concise_mode}{true}

\NewDocumentCommand{\card}{sO{}m}{%
  {\IfBooleanTF{#1}
    {\oldnormaux{\left|\right.}{\left.\right|}{#3}}
    {\oldnormaux{#2|}{#2|}{#3}}}
}
\makeatletter
\newcommand{\oldnormaux}[3]{\mathpalette\oldnormaux@i{{#1}{#2}{#3}}}
\newcommand{\oldnormaux@i}[2]{\oldnormaux@ii#1#2}
\newcommand{\oldnormaux@ii}[4]{%
  \sbox\z@{$\m@th#1#2#4#3$}%
  \sbox\tw@{$\m@th\|$}%
  \mathopen{\hbox to\wd\tw@{\hss\vrule height \ht\z@ depth \dp\z@ width .2\wd\tw@\hss}}%
  #4
  \mathclose{\hbox to\wd\tw@{\hss\vrule height \ht\z@ depth \dp\z@ width .2\wd\tw@\hss}}%
}
\makeatother

\makeatletter
\newcommand{\longdash}[1][2em]{%
  \makebox[#1]{$\m@th\smash-\mkern-7mu\cleaders\hbox{$\mkern-2mu\smash-\mkern-2mu$}\hfill\mkern-7mu\smash-$}}
\makeatother
\newcommand{\omitskip}{\kern-\arraycolsep}

\IEEEoverridecommandlockouts

\begin{document}

\title{Splat-Nav: Safe Real-Time Robot Navigation in Gaussian Splatting Maps}

\author{Timothy Chen$^{1\star}$, Ola Shorinwa$^{1\star}$, Joseph Bruno$^3$, Aiden Swann$^1$, Javier Yu$^1$, \\ Weijia Zeng$^2$, Keiko Nagami$^1$, Philip Dames$^3$, Mac Schwager$^1$
\thanks{$^{\star}$ The co-first authors contributed equally.}%
\thanks{$^{\dagger}$ This work was supported in part by ONR grant N00014-23-1-2354 and DARPA grant HR001120C0107 and NSF grant 2220866. Toyota Research Institute provided funds to support this work. T. Chen was supported by a NASA NSTGRO Fellowship and A. Swann was supported by NSF GRFP DGE-2146755.}
\thanks{$^{1}$ Stanford University, Stanford, CA 94305, USA
        {\tt\small \{chengine, shorinwa, swann, javieryu, knagami, schwager\}@stanford.edu}}%
\thanks{$^{2}$University of California San Diego, San Diego, CA 92093, USA,                {\tt\small wez195@ucsd.edu}}%
\thanks{$^{3}$Temple University, Philadelphia, PA 19122, USA,                {\tt\small \{brunoj6, pdames\}@temple.edu}}%
        }

\maketitle

\begin{abstract}

We present Splat-Nav, a real-time robot navigation pipeline for Gaussian Splatting (GSplat) scenes, a powerful new 3D scene representation. 
Splat-Nav consists of two components: 1) Splat-Plan, a safe planning module, and 2) Splat-Loc, a robust vision-based pose estimation module.
Splat-Plan builds a safe-by-construction polytope corridor through the map based on mathematically rigorous collision constraints and then constructs a B\'ezier curve trajectory through this corridor.
Splat-Loc provides real-time recursive state estimates given only an RGB feed from an on-board camera, leveraging the point-cloud representation inherent in GSplat scenes. 
Working together, these modules give robots the ability to recursively re-plan smooth and safe trajectories to goal locations.  Goals can be specified with position coordinates, or with language commands by using a semantic GSplat. We demonstrate improved safety compared to point cloud-based methods in extensive simulation experiments.  In a total of 126 hardware flights, we demonstrate equivalent safety and speed compared to motion capture and visual odometry, but without a manual frame alignment required by those methods. We show online re-planning at more than \unit[2]{Hz} and pose estimation at about \unit[25]{Hz}, an order of magnitude faster than Neural Radiance Field (NeRF)-based navigation methods, thereby enabling real-time navigation. We provide experiment videos on our project page at \url{https://chengine.github.io/splatnav/}. Our codebase and ROS nodes can be found
at \url{https://github.com/chengine/splatnav}.

\end{abstract}

\IEEEpeerreviewmaketitle



\section{Introduction}
\label{sec:intro}

Autonomous robotic operation requires robots to localize themselves within an envrionment, plan safe paths to reach a desired goal location, and have closed-loop trajectory-tracking.
Traditionally, the fundamental problems of planning and localization have been performed in maps represented as occupancy grids \cite{elfes1989using}, triangular meshes \cite{edelsbrunner2003surface}, point clouds \cite{kim2018slam}, and Signed Distance Fields (SDFs) \cite{osher2003level}, all of which provide well-defined geometry.

However, these explicit scene representations are generally constructed at limited resolutions (to enable real-time operation), leaving out potentially-important scene details that could be valuable in planning and localization problems.

\begin{figure*}[t!]
\centering
\includegraphics[width=\textwidth]{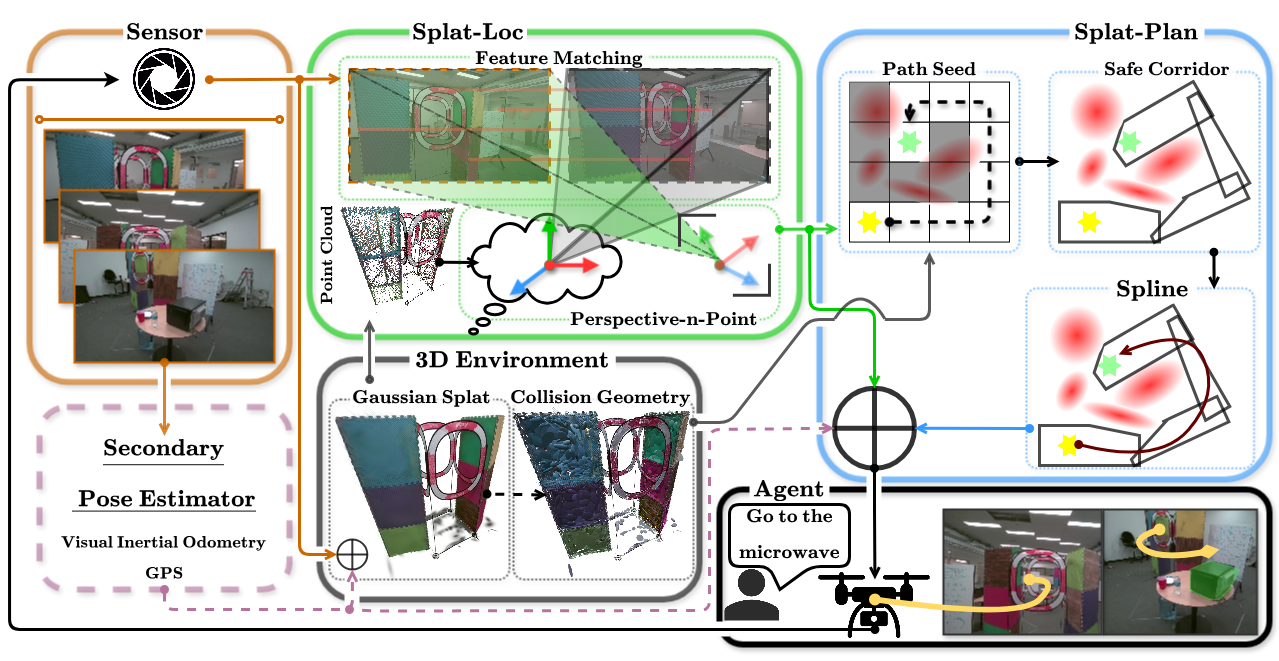}
\caption{Splat-Nav, consists of a safe planning module, Splat-Plan, and robust localization module, Splat-Loc, both operating on a Gaussian Splatting environment representation. In Splat-Plan we develop a fast, new ellipsoid-ellipsoid collision test to find a safe flight corridor through the GSplat, and plan a spline through the corridor. In Splat-Loc we localize the robot using only RGB images through a PnP algorithm, using the GSplat to render a point cloud. We use a language-embedded GSplat to enable open-vocabulary specification of goal locations like ``go to the microwave.'' 
} 
\label{fig:pipeline}
\end{figure*}

Neural Radiance Fields (NeRFs) \cite{mildenhall2020a} have recently been used to implicitly represent $3$D scenes. NeRFs consist of a volumetric density field and a view-dependent color field parameterized by multilayer perceptrons (MLPs). NeRFs generate photo-realistic scene reconstructions, addressing the fundamental limitations of explicit representations; however, NeRFs require running inference on a deep neural network to render the scene, making them impractical for real-time use in robotic path planning. 
More recently, Gaussian Splatting (GSplat) \cite{kerbl3Dgaussians} has emerged as a viable scene representation compared to NeRFs, representing the environment with Gaussian (ellipsoidal) primitives. Compared to NeRFs, GSplats generate higher-fidelity maps at faster rendering rates, with shorter or comparable training times. More importantly for robotics, GSplats, unlike NeRFs, offer a geometrically consistent collision geometry, enabling us to use level sets of these Gaussians to generate an ellipsoidal representation of the scene. These interpretable geometric primitives facilitate the development of rigorous motion planning algorithms that are safe, robust, and real-time. 

In this paper, we introduce \emph{Splat-Nav}, a pipeline for drone navigation in GSplat maps with a \emph{monocular} camera. Splat-Nav comprises a lightweight pose estimation module, Splat-Loc, coupled with a planning module, Splat-Plan, to enable safe navigation from RGB-only (monocular) camera observations, as illustrated in \Cref{fig:pipeline}.
Given an incoming RGB frame, \mbox{Splat-Loc} performs Perspective-n-Point (PnP)-based localization, leveraging the GSplat map to estimate the RGB and depth values rendered at candidate poses, which are then used to estimate the drone's pose.
Next, Splat-Plan ingests the estimated pose computed by \mbox{Splat-Loc} to generate an initial trajectory, which is subsequently optimized to lie within safe flight corridors constructed from the ellipses that make up the GSplat map. The trajectory is parametrized by smooth, continuously safe B\'ezier splines that route the robot to a specified position or to a open-vocabulary language-conditioned goal location (i.e., ``go to the microwave''). This feature enables the execution of Splat-Nav in a wide array of deployment conditions, such as in search missions where the precise location of targets is not known. 

Additionally, the proposed system enables both open-loop trajectory generation and closed-loop re-planning. The latter is important in long trajectories, where existing onboard localization may drift or be subject to noise, impacting the overall safety of the executed trajectory of the robot. In these scenarios, Splat-Loc estimates can either be fused with that of the existing localization module or used as a correction mechanism to steer the current motion toward a safer one. Finally, closed-loop re-planning additionally enables changes in goal locations during execution, leading to more dynamic plans.

In extensive simulations we compare Splat-Plan and Splat-Loc with baseline alternatives for planning and localization, respectively.
We show Splat-Plan is always safe with respect to the full collision geometry, while four variants of a point-cloud based planner sometimes lead to collisions, or fail to find trajectories. Splat plan achieves similar or better solutions in terms of path length compared to point cloud-based planner in all cases, with similar computation time. Splat-Plan runs at no less than \unit[2]{Hz}; comparable to point cloud-based solutions for the same scenes, but faster than gradient-based NeRF planners \cite{adamkiewicz2022vision} and sampling-based planners (greater than \unit[1]{Hz}) for similar solution quality. 
Similarly, we find that Splat-Loc is more accurate, faster, and fails less often compared to baselines. We demonstrate online pose estimation at about \unit[25]{Hz} on a desktop computer, enabling real-time navigation.

Finally, in an experimental campaign with 124 hardware flights, we show that Splat-Nav (Splat-Plan and Splat-Loc running together) perform as well as motion capture or on board VIO, without the manual frame alignment required for those methods to align the MoCap or VIO frame with the GSplat (since both Splat-Plan and Splat-Loc operate natively in the same GSplat map). 

The key contributions of this paper are as follows:
\begin{itemize}
    \item We develop a fast polytope corridor generation algorithm to enable provably safe planning for drone navigation in GSplat maps.
    \item We develop a fast camera localization module based on GSplat maps that does not require manual alignment of the pose estimation frame to the planning frame, improving the synergy between planning and pose estimation.
    \item We demonstrate safe closed-loop re-planning with open-vocabulary goal specification, across a series of $124$ hardware experimental trials.
\end{itemize}

\section{Related Work}
\label{sec:related}
We review the related literature in robot planning and localization with different map representations, categorized into three groups: traditional representations (e.g., occupancy grids, meshes, point clouds, and SDFs), NeRFs, and GSplat.


\smallskip
\noindent \textbf{Planning.}
We refer readers to \cite{lavalle2006planning} for an excellent explanation of planning algorithms in robotics. Most relevant for this work are the graph-based planners (e.g., A$^{\star}$), which compute a path over a grid representation of the environment; sampling-based planners (e.g., PRM \cite{kavraki1996probabilistic}, RRT 
and RRT$^{\star}$ \cite{karaman2011anytime}), which generate a path by sampling candidate states within the configuration space of the robot; and trajectory optimization-based planners (e.g., CHOMP \cite{ratliff2009chomp} and Traj-Opt \cite{schulman2014motion}), which take an optimization-based approach to planning. Prior work in \cite{liu-safe-flight-corridors} utilizes an optimization-based approach in path planning, taking a point cloud representation of the environment and converting this into a set of safe polytopes. The resulting safe polytopes are utilized in computing a safe trajectory, parameterized as a spline from a quadratic program (QP), which can be efficiently solved.

There is also extensive literature on planning based on onboard sensing. Typically, these works present reactive control schemes \cite{van2008reciprocal, perceptionaware},
using the sensed depth to perform collision checking in real time. These methods typically are myopic, reasoning only locally about the scene. Consequently, such methods often converge to local optima, preventing the robot from reaching its goal, especially in cluttered, complex environments. An alternate approach is to construct a map of the environment using depth measurements from Lidar or RGB-D sensors. Often, a Signed Distance Field (SDF) or its truncated variant (TSDF) is constructed from depth data \cite{voxblox, fiesta}, which is encoded within a voxel representation. Such a representation is typical in dynamic robotic motion planning, providing fast collision checking and gradients in planning; however, voxel-based scene models do no provide visually rich or geometrically detailed scene representations compared to NeRFs or GSplats. Point cloud and voxel-based representations require a significant number of points or voxels for high-fidelity scene reconstruction, increasing the computational burden. Prior work \cite{goel2021rapid, goel2023incremental} introduced a Gaussian Mixture Model (GMM) as a more effective scene representation, which preserves the accuracy of point cloud-based map representations without the additional computational overhead, enabling fast exploration of unknown environments by multi-robot teams.
Nevertheless, the aforementioned methods do not achieve photorealistic scene rendering.

More recent research has developed planning methods for highly expressive neural representations, such as NeRFs, which represent the environment as a spatial density field (with color) \cite{mildenhall2020a}. Using a NeRF map, NeRF-Nav \cite{adamkiewicz2022vision} plans trajectories that minimize the total collision cost for differentially flat robots, e.g., quadrotors. Further, CATNIPS \cite{chen2023catnips} converts the NeRF into a probabilistic voxel grid and then uses this to generate trajectories parameterized as B\'ezier curves. The work in \cite{tong2023enforcing} uses the predicted depth map at sampled poses to enforce step-wise safety using a
control barrier function. The above works are complementary, with \cite{adamkiewicz2022vision, chen2023catnips} serving as high-level planners that encourage non-myopic behavior, while \cite{tong2023enforcing} can be used as a safety filter for a myopic low-level controller.
GSplats are faster to train and provide higher fidelity visual and geometric detail compared to NeRFs \cite{kerbl3Dgaussians}, making them a strong candidate for scene representations for robot planning.  To the best of our knowledge, our work is the first to propose a planning algorithm suitable for GSplat scene representations.



\smallskip
\noindent \textbf{Localization.}
Prior work in robot localization typically utilizes filtering schemes, such as Extended Kalman Filters (EKFs) \cite{guibas1997robot, eman2020mobile}, Particle Filters (PFs) \cite{fox2001particle, zhang2019improved}, and other related filters \cite{ullah2019localization, biswas2012depth}, to solve the pose localization problem. These methods generally estimate the pose of the robot from low-dimensional observations (measurements), extracted from the high-dimensional observations collected by the robot's onboard sensors, such as cameras. This approach often fails to leverage the entire information available in the raw, high-dimensional measurements. Learning-based filtering methods \cite{karkus2018particle, jonschkowski2018differentiable}
seek to address this limitation using deep learning to develop end-to-end frameworks for localization, computing a pose estimate directly from raw camera images. Although learning-based approaches can be quite effective given sufficient training data, these methods are often limited to a single robot platform (dynamics model) and thus require separate filters for each robot or environment.

There is some existing work on tracking the pose of a robot equipped with an on-board camera and IMU through a pre-trained NeRF map. For pose localization, these methods compute a pose that minimizes the photometric loss, given an initial guess of the camera's pose. iNeRF \cite{yen2021inerf} does this for single images, and NeRF-Nav \cite{adamkiewicz2022vision} and Loc-NeRF \cite{loc-nerf} both track a trajectory using a sequence of images. Other works consider simultaneous localization and mapping (SLAM) using a NeRF map representation. Existing methods such as \cite{yen2021inerf, nice-slam}
all simultaneously optimize the NeRF weights and the robot/camera poses. NeRF-SLAM \cite{nerf-slam} proposes a combination of an existing visual odometry pipeline for camera trajectory estimation together with online NeRF training for the 3D scene.
Although applicable to localization in Gaussian Splatting, photometric loss-based localization methods generally have a small region of convergence and require multiple passes through the scene representation for gradient computation, leading to increased computation times. In this work, we introduce a localization algorithm based on the \mbox{perspective-n-point} problem, which addresses these challenges.


There are a few recent works on SLAM using a GSplat representation of the environment \cite{yan2023gs, yugay2023gaussian, matsuki2023gaussian}. 
These SLAM methods use the photometric loss to optimize the camera's pose, suffering from the aforementioned challenges, which we address with our proposed method. Moreover, these SLAM methods do not consider safe trajectory planning and control, as is the focus of this paper.

\section{3D Gaussian Splatting}
\label{sec:background}


{
\color{black}
\smallskip
\noindent \textbf{Background.}
We present a brief introduction to $3$D Gaussian Splatting \cite{kerbl3Dgaussians}, a radiance field method for deriving volumetric scene representations from a set of monocular images. 
Gaussian Splatting represents non-empty space in a scene using $3$D Gaussian primitives, each of which is parameterized by a mean ${\mu \in \mbb{R}^{3}}$ (defining its position), covariance matrix ${\Sigma \in \mbb{S}_{++}}$ (related to its spatial extent and orientation), opacity ${\alpha \in [0, 1]}$, and spherical harmonics (SH) coefficients (defining view-dependent colors). 
The scene is typically initialized using a sparse point cloud computed via structure-from-motion \cite{schoenberger2016sfm}.
To render an image from a given camera pose, the $3D$ Gaussians are projected onto the image plane using an affine approximation of the projective transformation, given by ${\Sigma_{2D} = J W \Sigma W^{T}J^{T}}$, with Jacobian $J$ and viewing transformation $W$. 
The number of primitives, along with the coefficients for each primitive, is then learned via stochastic gradient descent with a loss function comprising of the photometric loss between the rendered and ground-truth images and the structural similarity (SSIM) index loss (the same as NeRF methods). 

For better numerical optimization, the anisotropic $3$D covariance of each Gaussian is written as: ${\Sigma = RSS^{T}R^{T}}$, where ${R \in \SO(3)}$ is a rotation matrix (parameterized by a quaternion) and $S$ is a diagonal scaling matrix (parameterized by a $3$D vector). 
This anisotropic covariance along with adaptive density control (i.e., splitting and merging Gaussians) enable the computation of compact high-quality representations, even in complex scenes, unlike many state-of-the-art point-based rendering methods. Further, $3$D Gaussian Splatting obviates the need for volumetric ray-marching required in NeRF methods, enabling high-quality real-time rendering, even from novel views.


\smallskip
\noindent \textbf{GSplats versus NeRFs.}
Gaussian Splatting typically requires less training time than state-of-the-art NeRF methods, while achieving about the same or better photometric quality. The biggest difference is in the rendering speed, where Gaussian Splatting achieves real-time performance \cite{kerbl3Dgaussians}. Moreover, $3$D Gaussian Splatting enables relatively fast extraction of a mesh representation (\cref{remark:gsplat_extent}) of the scene from the Gaussian primitives, and instantaneous extraction of the primitives themselves. In contrast, slower meshing techniques \cite{lorensen1998marching} are needed for NeRFs, and the extraction of a point cloud requires slow volumetric rendering of many training viewpoints. In \cref{fig:viz_nerf_vs_gs}, we visualize the ground-truth mesh, the GSplat mesh, and the associated point cloud extracted from a NeRF of a simulated Stonehenge scene to showcase the collision geometry quality of GSplats over NeRFs. Quantitatively, the GSplat mesh has a smaller Chamfer distance (0.031 with 3M vertices) compared to the NeRF point cloud (0.081 with 4M points) despite having fewer points. We note that the NeRF does not necessarily yield a view-consistent geometry due to volumetric rendering, especially when the point cloud is not post-processed to remove outliers, leading to relatively poor collision geometry despite having good photometric quality. 

\begin{remark}
\label{remark:gsplat_extent}
    The original work \cite{kerbl3Dgaussians} only projects $3D$ Gaussians whose $99\%$ confidence interval intersects the view frustum of a camera, effectively restricting the scene representation to the $99\%$ confidence ellipsoid associated with each Gaussian. Consequently, the union of the $99\%$ confidence ellipsoids represents the entirety of the geometry of the scene learned during the training procedure. We find that this cutoff is too conservative, due to the fact that the color of the Gaussians toward the tails of the distribution are close to transparent. Instead, we find that renderings of the $1\sigma$ collision geometry closely matches that of the GSplat depth channel, so we elect to use $1\sigma$-ellipsoid as the collision geometry for the remainder of this work. Future work will seek to explore the calibration of this cutoff. 
\end{remark}

\begin{figure}[h]
    \centering
    \includegraphics[width=\columnwidth, trim={5cm 0.5cm 0 0},clip]{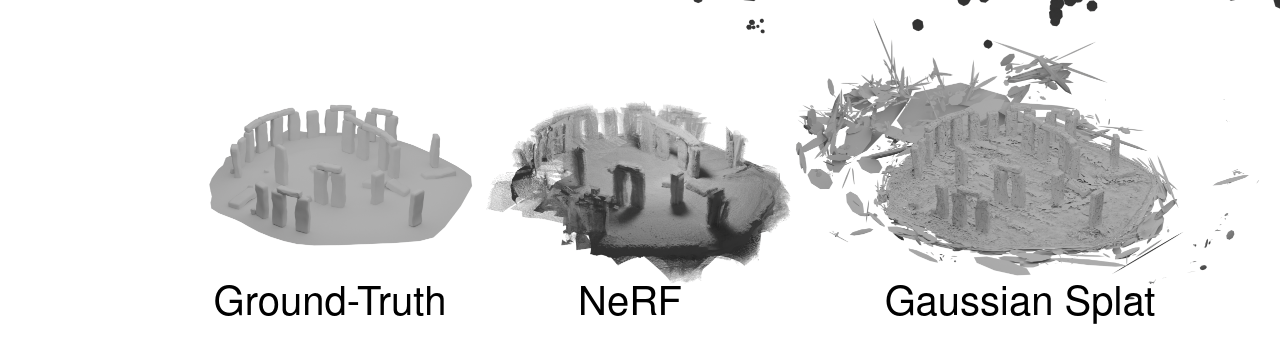}
    \caption{Visualization of a point cloud from a NeRF and a mesh from a Gaussian Splat in the synthetic scene Stonehenge. The Chamfer Distance between the NeRF and ground-truth is 0.081 (with 4M points). The Chamfer distance between the GSplat and ground-truth is 0.031 (with 3M vertices). The collision geometry (especially the foreground) of the GSplat is better and can be extracted instantaneously from the model parameters compared to the costly rendering procedure from many viewpoints to create a point cloud from the NeRF.}
    \label{fig:viz_nerf_vs_gs}
\end{figure}

\smallskip
\noindent\textbf{Semantic Gaussian Splatting.} To enable goal specifications for the navigation task in natural-language, we leverage semantic Gaussian Splatting \cite{qin2024langsplat, zhou2024feature, shorinwa2024fast}, which distills $2$D language semantics from vision-language models, e.g., CLIP \cite{radford2021learning}, into $3$D GSplat models. In general, these methods assign learnable semantic codes to each Gaussian, supervised by the robust semantic features extracted by $2$D foundation models. The semantic GSplats are trained in the same way as non-semantic GSplat via gradient descent. Semantic Gaussian Splatting has been utilized in prior work to enable open-vocabulary robotics tasks, e.g., robotic manipulation \cite{ji2024graspsplats, shorinwa2024splat}.

In the subsequent sections, we present the core contributions of our work in deriving an efficient navigation pipeline for robots, describing how we leverage $3$D Gaussian Splatting as the underlying scene representation. Specifically, the quick extraction of simple convex primitives (whose union closely approximates the ground-truth scene geometry) promotes the development of guarantees on safety and solution quality of Splat-Plan and facilitates real-time deployment with low sim-to-real gap while navigating in Gaussian Splatting environments. Similarly, the fast and high-quality color and depth rendering from arbitrary viewpoints of the GSplat enables robust, fast camera localization in Splat-Loc.





\section{Planning with Safe Polytopes}
\label{sec:planning}
Now, we present Splat-Plan, our planner for GSplat maps. Splat-Plan generates safe polytopic corridors (inspired by \cite{liu-safe-flight-corridors}) that represent the free space of a GSplat map between an initial configuration to a goal configuration. These corridors, and the resulting trajectories through them, are rigorously built on theory derived from tests for intersection between ellipsoids. The method is fast enough to provide real-time operation, provides safety guarantees extending to any scene with a pre-trained GSplat representation, and is not overly-conservative.  

We stress that, as with any safety guarantee on a map, our ultimate safety rests on the completeness of the map.  If the map does not reflect the presence of an obstacle, our method may collide with the obstacle---we cannot avoid what we cannot see.  In practice, we observe that GSplat maps provide fast and efficient representations of the underlying ground-truth geometry, as validated in our hardware experiments.

We would also like to motivate the use of the full collision geometry of GSplats for planning compared to conventional representations like point clouds in an RGB setting. It is common to extract the means of the GSplat to form a point cloud. However, in feature-less regions, we observe that the point cloud can be quite sparse. Meanwhile, the full collision geometry spanned by the ellipsoids covers the full surface. This phenomenon can be observed in \cref{fig:pipeline}, where the render of the ellipsoidal representation of the collision geometry closely mimics the RGB render from the GSplat. However, the point cloud extracted from the means is very sparse. While usable for localization, such a sparse representation leaves a large sim-to-real gap when planning safe trajectories close to those areas. Another option is to sample the surface of these primitives for a point cloud, but even with this modification, point cloud-based planners are not as robust as Splat-Plan (\cref{sec:results}).

Before presenting the planning problem, we make the following assumptions on the representations of the robot $\mcal{R}$ and the map $\mcal{G}$ considered in this work. We assume that the robot is represented by a union of the ellipsoids in the non-empty set $\{\mcal{E}_{\mcal{R}, i}\}_{i = 1}^{d}$, where $d$ denotes the cardinality of the set, 
i.e., ${\mcal{R} \subseteq \cup_{i=1}^d \mcal{E}_{\mcal{R}, i}}$.
For simplicity, we consider a singleton set $\mcal{E}_{\mcal{R}}$, noting that the subsequent discussion applies directly to the non-singleton case by running the collision check for all robot ellipsoids.
One can also convert a mesh or point cloud of a robot to an ellipsoid by finding the minimal bounding ellipsoid (or sphere).

We represent non-empty space in the environment with $\gamma\%$ confidence ellipsoids obtained from the GSplat map, as discussed in Remark~\ref{remark:gsplat_extent}, given by:
\begin{equation}
    \label{eq:environment_ellipsoid}
    \mathcal{E}_j = \{x \in \R^3 \mid (x - \mu_j)^T \Sigma_j^{-1} (x - \mu_j) \leq \chi_{3}^{2}(\gamma) \},
\end{equation}
where ${\mu_{j} \in \mbb{R}^{3}}$ denotes the mean of Gaussian $j$, ${\Sigma_{j} \in \mbb{S}_{++}}$ denotes its covariance matrix, and ${\chi_{3}^{2}(\gamma)}$ denotes the $\gamma$th percentile of the chi-square distribution with three degrees of freedom. The union of these ellipsoids, given by ${\mathcal{G} = \{\mathcal{E}_j\}_{j=1}^{N}}$, defines the occupied space in the environment. To simplify notation, we express the ellipsoid in \eqref{eq:environment_ellipsoid} in standard form: ${\mathcal{E}_j = \{x \in \R^3 \mid (x - \mu_j)^T \Sigma_j^{-1} (x - \mu_j) \leq 1 \}}$, where we overload notation with ${\Sigma_j \coloneqq \chi_{3}^{2}(\gamma) \Sigma_j}.$ Based on Remark \ref{remark:gsplat_extent}, we set $\gamma = 0.2$ to be safe with respect to the entirety of the supervised scene.

\begin{remark}[Online Gaussian Splatting]
    Our planning algorithm requires a GSplat map. While this map can be trained online using real-time SLAM methods for radiance fields \cite{matsuki2023gaussian, yugay2023gaussian, yan2023gs}, which is a very new and active area of research, we limit the scope of this work to only plan in pre-trained maps. 
\end{remark}

\begin{remark}[Handling Uncertainty of the Scene Representation]
\label{rem:uncertainty}
    We can vary the value of ${\gamma}$ (from that used during the training procedure) based on the quality of the GSplat map and uncertainty in different regions of the GSplat map. In general, larger values of $\gamma$ inflate the volume of the confidence ellipsoids associated with each Gaussian, resulting in greater safety margins and more conservative planning. The converse holds if smaller values of $\gamma$ are selected. Moreover, for simplicity, we utilized a uniform value of $\gamma$. However, the value of $\gamma$ can vary among the ellipsoids, allowing the planner to account for varying levels of uncertainty in different regions of the GSplat map. Likewise, the volume of the ellipsoid representing the robot can be increased/decreased to account for uncertainty in the pose of the robot.
\end{remark}

\begin{remark}[Dynamic Scenes]
We limit our discussion to planning in static scenes. However, we note that our method readily applies to planning in dynamic scenes, under the assumption that a dynamic Gaussian Splatting scene representation can be constructed. We discuss more about planning in dynamic scenes in Section~\ref{sec:limitations_future_work}.
\end{remark}

\smallskip
\noindent \textbf{Problem Statement.}
Given a bounding ellipsoid $\mathcal{E}_\mathcal{R}$ for the robot and a GSplat map $\mcal{G}$, we seek to find a smooth, feasible path $x(t)$ for a robot to navigate from an initial configuration $x(0) = x_0$ to a specified goal configuration $x(T) = x_f$, such that there are no collisions in the continuum, i.e., ${\mathcal{E}_\mathcal{R}(x(t)) \cap \mathcal{E}_j = \emptyset, \, \forall \mathcal{E}_j \in \mathcal{G}}$, $\forall \, t \in [0, T]$.

\smallskip
\noindent \textbf{Collision Detection.}
\label{ssec:collision}
We leverage the ellipsoidal representations of the robot and the environment to derive an efficient collision-checking algorithm, based on \cite{ellipsoidtest}, where we take advantage of GPU parallelization for faster computation. We build upon \cite{ellipsoidtest} rather than on other existing ellipsoid-to-ellipsoid intersection tests, because of its amenability to significant GPU parallelization. We do not utilize the GJK algorithm \cite{gilbert1988fast}, since we do not require knowledge of the distance between the two ellipsoids. For completeness, we restate the collision-checking method from \cite[Proposition 2]{ellipsoidtest}.

\begin{theorem} \label{thm:ellipsoidtest1}
Given two ellipsoids $\mathcal{E}_a, \mathcal{E}_b$ (with means $\mu_a, \mu_b$ and covariances $\Sigma_a, \Sigma_b$) and the concave function ${K: (0, 1) \to \R}$,
\begin{equation*}
    K(s) = (\mu_b-\mu_a)^T \left[ \frac{1}{1-s}\Sigma_a + \frac{1}{s}\Sigma_b \right] ^{-1}(\mu_b-\mu_a),
\end{equation*}
$\mathcal{E}_a \cap \mathcal{E}_b = \emptyset$ if and only if there exists $s\in (0, 1)$ such that $K(s) > 1$.
\end{theorem}

We will let $\Sigma_{a,b}(s) = \frac{1}{1-s}\Sigma_a + \frac{1}{s}\Sigma_b$ for compactness throughout the remainder of the work. Using this, we can rewrite $K(s) = (\mu_b-\mu_a)^T \Sigma_{a,b}^{-1}(s) (\mu_b-\mu_a)$.

Note that $K(s)$ is concave in $s$ and convex with respect to the means and variances. Theorem \ref{thm:ellipsoidtest1} is a complete test that will always indicate whether two ellipsoids are in collision or not. We note, however, that solving the feasibility problem in Theorem~\ref{thm:ellipsoidtest1} can be challenging, particularly in large-scale problems, where the feasibility problem has to solved for many pairs of ellipsoids with an associated matrix inversion procedure in each problem. 
In general, GSplat environments consists of hundreds of thousands to millions of Gaussians \cite{kerbl3Dgaussians}.
Consequently, we eliminate the matrix inversion by operating in a shared basis for both $\Sigma_A$ and $\Sigma_B$, for faster collision-checking, detailed in the following Proposition.

\begin{proposition}
    \label{prop:ellipsoidtest2}
    By solving the generalized eigenvalue problem for $\Sigma_a$ and $\Sigma_b$, we obtain generalized eigenvalues $\lambda_i$ and the corresponding matrix of generalized eigenvectors $\phi$. 
    Then
    \begin{equation*}
        \Sigma_{a,b}^{-1}(s) = \phi \, \diag \left(\frac{s (1 - s)}{1 + s(\lambda_i - 1)} \right) \phi^T.
    \end{equation*}

\end{proposition}

\begin{proof}
     We present the proof in Appendix~\ref{app:prop_gen_eigen_proof}.
\end{proof}

\begin{corollary}
\label{cor:gs_intersection_test}
Given a GSplat representation with $\Sigma_j = R S S^T R^T$, let $S S^T = \diag\left(\lambda_i^2 \right)$. If we choose to parameterize our robot body as a sphere with covariance ${\Sigma_\mathcal{R} = \kappa^2 \mathcal{I}}$, then
\begin{equation*}
    \Sigma_{a,b}^{-1}(s) = R \, \diag \left( \frac{s (1 - s)}{\kappa^2 + s(\lambda_i^2 - \kappa^2)} \right) R^T.
    \label{eq:gs_ellipsoid_test}
\end{equation*}
\end{corollary}

From hereafter, we will treat \cref{cor:gs_intersection_test} as the general formula expression of $K(s)$ for both sphere and ellipsoid to ellipsoid, substituting the appropriate variables for $R, \kappa,$ and $\lambda$. For readers interested in visualizing the behavior of the aforementioned proposition and corollary, we direct readers to \cite[Figure 2]{gilitschenski2012robust} to understand how the shape of $K(s)$ changes as ellipsoids move through space. 

\smallskip
\noindent \textbf{Extension to Linear Motion.}
\cref{prop:ellipsoidtest2} (and \cref{cor:gs_intersection_test}) can be extended to account for linear motion of ellipsoidal bodies. 
Consider a line segment starting at point $x_0$ and ending at point $x_1$, and let $\delta_x = x_1 - x_0$. Then the line segment can be parameterized as $\ell(t) = x_0 + t \delta_x$ for $t \in [0, 1]$. 
In our case, we consider a moving ellipsoid $\mathcal{E}_a$, which starts at $x_0 = \mu_a$. Let
\begin{equation}
    K(s, t) = (\mu_a + t \delta_x - \mu_b)^T \Sigma_{a,b}^{-1}(s) (\mu_a + t \delta_x - \mu_b).
\end{equation}
The ellipsoid $\mathcal{E}_a$ must satisfy \cref{prop:ellipsoidtest2} at all points along the line $\ell(t)$, so the safety test\footnote{Note that the sliding of the ellipsoid along a line forms capsules, making \cref{corr:ellipsoid-test-line} also a necessary and sufficient collision test between this type of geometry with an ellipsoid.} is
\begin{equation}
\label{eq:general_line_test}
    \min_{t \in [0, 1]} \max_{s \in (0, 1)} K(s, t) > 1.
\end{equation}

We seek to solve a portion of \cref{eq:general_line_test} using \cref{corr:ellipsoid-test-line}. 

\begin{corollary}
    \label{corr:ellipsoid-test-line}
    Note that $K(s,t)$ is a convex scalar function in $t$ because $\delta_x^T \Sigma_{a,b}^{-1}(s) \delta_x > 0$ for all $\delta_x \neq 0$ since covariance matrices are symmetric and positive definite.
    The $t$ that minimizes $K(s,t)$ is
    \begin{equation*}
        \hat{t}(s) = -\frac{(\mu_a - \mu_b)^T \Sigma_{a,b}^{-1}(s) \delta_x}{\delta_x^T \Sigma_{a,b}^{-1}(s) \delta_x},
    \end{equation*}
    so the optimal value will be $t^*(s) = \textrm{clamp}(\hat{t}(s), 0, 1)$.
    Then we can write the safety check as 
    \begin{equation*}
        K(s) = K(s, t^*(s)) = K(s, 0) + (\mu_a - \mu_b)^T \Sigma_{a,b}^{-1}(s) \delta_x t^*(s)
    \end{equation*}
\end{corollary}
\begin{proof}
Sion's minimax theorem states that switching the order of the minimum and maximum yields identical solutions when $K(s, t)$ is concave in $s$ and convex in $t$. Additionally, $s$ must lie in a convex set and $t$ in a compact, convex set, which \cref{eq:general_line_test} admits. Consequently, the max-min problem results is an inner minimization problem of a quadratic, which is solved in closed-form. The point-wise minimum of concave functions $K(s, t^*(s))$ is concave, hence the outer maximization is still over a concave function.
\end{proof}

As a byproduct of concavity of $K(s)$, we have the following corollary:

\begin{corollary}
\label{cor:intersection_test_sampling}
By concavity of $K(s)$, any approximate solution $\hat{s}$ in  Theorem~\ref{thm:ellipsoidtest1}, \cref{prop:ellipsoidtest2}, \cref{cor:gs_intersection_test}, or \cref{corr:ellipsoid-test-line} results in $K(s^*) \geq K(\hat{s})$. Hence, no approximate solution will yield false negatives (i.e., miss a collision).
\end{corollary}

\smallskip
\noindent \textbf{Optimization.}
While \cref{cor:intersection_test_sampling} is a nice blanket certificate, we can craft approximate solutions that exponentially converge to the optimal $s^*$ such that false \emph{positives} tend to 0 (i.e., a perfect approximation). Bisection searches (especially in $1$D) are efficient, simple ways to guarantee exponential convergence for bounded variables in smooth convex/concave functions, i.e., $||\hat{s}_i - s^*|| \leq \epsilon$ for any desired $\epsilon$. We propose to solve $\max_{s \in [0, 1]} K(s)$ using \cref{alg:bisection-search}.
\begin{algorithm2e}[th]
    \caption{$K(s)$ Bisection Search}
    \label{alg:bisection-search}
    
    \KwIn{number of iterations $k$\;}
    \KwOut{maximal estimator $\hat{s}$\;}
    \tcp{Initialize lower and upper bounds}
    $s_l \gets 0$, $s_h \gets 1$\;
    \For{$i\gets 0$ \KwTo $k$}{
    \tcp{Test midpoint}
    $\hat{s}_i \gets \frac{s_l + s_h}{2}$\;
    \tcp{Find minimal $t$}
    $t_i^* \gets$ clamp$\left(-\frac{\delta_x^T \Sigma_{a,b}^{-1}(\hat{s}_i) (x_0 - \mu)}{\delta_x^T \Sigma_{a,b}^{-1}(\hat{s}_i) \delta_x}, 0, 1\right)$\;
    \tcp{Find which way to move s}
    \uIf{$\nabla_s K(\hat{s}_i, t_i^*) \geq 0$}{
        $s_l \gets \hat{s}_i$\;
    }
    \uElse{
        $s_h \gets \hat{s}_i$\;
    }
    }
    $\hat{s} \gets \hat{s}_k$\;
\end{algorithm2e}

\begin{corollary}
    \label{corr:convergence_of_bisection}
The distance of $s$ to $s^*$ converges at a rate of $\epsilon = \frac{1}{2^{k+1}}$ through $k$ iterations using \cref{alg:bisection-search}. 
\end{corollary}

\begin{proof}
The bisection method guarantees convergence to a root of a continuous function $f(s)$ in the interval $[a, b]$ if $f(a)$ and $f(b)$ have different signs. The method achieves a rate of 
\begin{equation}
    ||s_k - s^*|| \leq \frac{||b - a||}{2^{k+1}}.
\end{equation}
Note that $K(s)$ evaluates to $0$ at both $s = \{0, 1\}$, and $K(s)$ is concave and non-linear. Therefore, we know a unique global maxima occurs between $0$ and $1$ and that the gradient $f(s) = \nabla_s K(s_k, t^*_k)$ is positive at $s = 0$ and negative at $s = 1$. 
\end{proof}

Additionally, note that \cref{alg:bisection-search} is batchable across many queries to different ellipsoids and is more efficient than performing uniform sampling for the same tolerance. For all tests, we use $k = 10$. 

\label{sub:safe_polytopes}

\smallskip
\noindent \textbf{Computing Safe Polytopes.}
We like to again emphasize that having convex primitives (ellipsoids) as an environment representation facilitates the development of interpretable algorithms for planning. This is especially true in the construction of safe trajectories within convex safe polytopes, which define obstacle-free regions of the robot's configuration space. We build upon prior work on convex decomposition of configuration spaces such as \cite{ruan2022efficient, liu-safe-flight-corridors}.
In this work, we leverage the ellipsoidal primitives to create polytopes that define the safe regions of space through the use of supporting hyperplanes. The ellipsoidal representation of the environment obtained from GSplat enables the direct computation of these convex obstacle-free regions without the need for a convex optimization procedure. 
Furthermore, our method is fast enough to run in real time. In the following proposition, we describe the generation of safe polytopes for a given robot.

\begin{proposition}
\label{prop:supp_hyperplane}
Given a seed point $x^*$ for a candidate robot position and a collision set $\mathcal{G}^* = \{\mathcal{E}_j\}$, a supporting hyperplane for the ellipsoid robot can be derived from Proposition \ref{prop:ellipsoidtest2} or \cref{cor:gs_intersection_test} for any desired buffer $\epsilon>0$:
\begin{equation*}
    \underbrace{\Delta^T_j \Sigma_{x^*,j}^{-1}(s^*)}_{a_j} x \geq \underbrace{(1+\epsilon){k^*_j} + \Delta^T_j \Sigma_{x^*,j}^{-1}(s^*) \mu_j}_{b_j},
\end{equation*}
where $\Delta_j = x^* - \mu_j, \Sigma_{x^*,j}^{-1}(s^*)$ uses the robot shape and $\Sigma_j$, and $({k_j^{*})^{2} = K(s^*) = \Delta^T_j \Sigma_{x^*,j}^{-1}(s^*) \Delta_j > 0}$, for $s^* \in (0, 1)$. By stacking the hyperplane constraints ($a_j, b_j$), we arrive at a polytope ${Ax \geq b}$ that is guaranteed to be safe.%
\end{proposition}

\begin{proof}
    We provide the proof in Appendix~\ref{app:proof_supp_hyperplane}.
\end{proof}

\begin{corollary}
    \label{corr:supp-hyperplane-line}
\cref{prop:supp_hyperplane} can be extended for the $K(s)$ in \cref{corr:ellipsoid-test-line} by substituting $x^* = x_0 + t^* \delta_x$, where $t^* \in [0, 1]$.
\end{corollary}

\begin{proof}
    We provide the proof in Appendix~\ref{app:proof-supp-hyperplane-line}.
\end{proof}

\cref{prop:supp_hyperplane} and \cref{corr:supp-hyperplane-line} guarantee manageability, coined by \cite{wang2024fastiterativeregioninflation}, which refers to the encapsulation of the seed object by the free-space partition. This property is important to guarantee connected-ness of each part of the safe flight corridor, which in turn admits a feasible trajectory that resides solely within the corridor. 

\begin{figure*}[th]
    \centering
    \includegraphics[width=\textwidth, trim={0 0em 0 0}, clip]{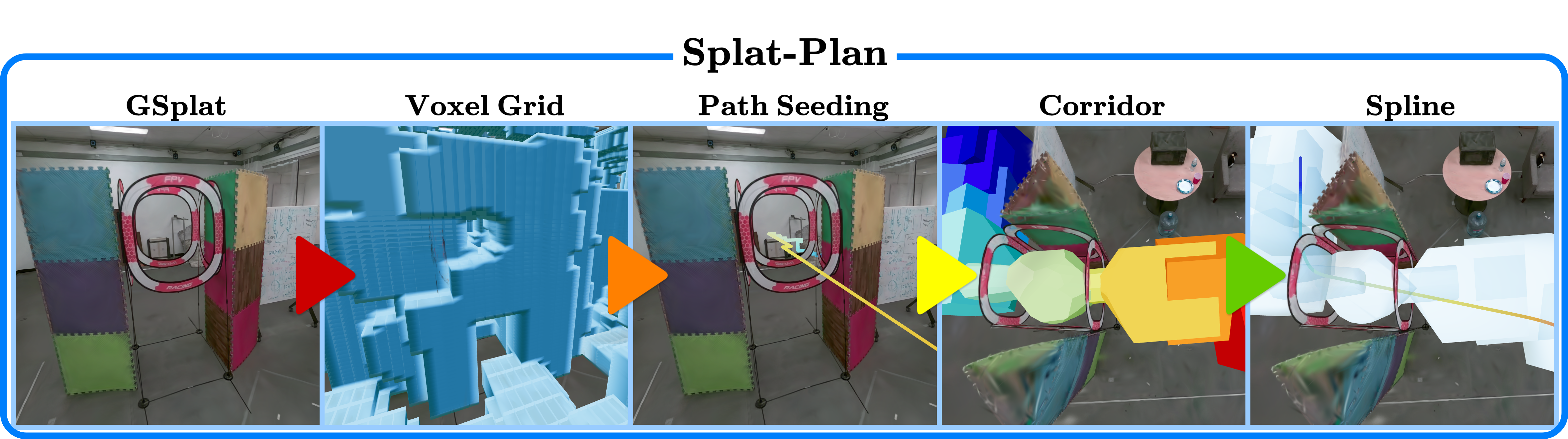}
    \caption{Splat-Plan, as described by \cref{alg:splat_plan}. Given a GSplat and its corresponding ellipsoidal collision geometry, Splat-Plan generates a binary occupancy grid representing the collision-less free space. Next, a seed path is created using graph-based search. \cref{corr:supp-hyperplane-line} synthesizes a set of connected polytopes forming a corridor around the feasible path. Finally, a quadratic program is solved for the control points of a sequence of B\'ezier curves that lives entirely within the corridor and hence is safe.}
    \label{fig:splatplan_explainer}
\end{figure*}

\smallskip
\noindent \textbf{Generating Safe Paths.}
\label{Sec:SafePlanning}
We present Splat-Plan, similar in spirit to the safe flight corridor method from \cite{liu-safe-flight-corridors}. There are four primary components: (1) feasible path seeding through graph-based search, (2) construction of a collision set around each part of the path, (3) generation of hyperplane constraints, and (4) smooth path planning posed as a spline optimization. For Splat-Plan, we leverage \cref{corr:ellipsoid-test-line}, \cref{alg:bisection-search}, and \cref{corr:supp-hyperplane-line} to generate safe polytopes along the seed path. Within the polytopes, we plan B\'ezier curves, which can be formulated as a quadratic program. 

\subsubsection{Seed Path}
There are two primary flavors of graph-based paths that are popular in the literature: those that use random trees (e.g. RRT) and those on uniform grids. We will detail how both can be used as an initialization.

Methods like RRT primarily rely on a module for collision detection at test points as well as a module to test for collision along a line. The use of \cref{corr:ellipsoid-test-line} serves both functions. Unfortunately, the probabilistic completeness of these algorithms make them undesirable for real-time execution. 

The use of a uniform grid to run algorithms like Dijkstra Search are optimal and typically faster than those of random trees \textbf{if} there exists a cheap subroutine that converts the scene representation into a uniform grid. Specifically, we would like to avoid expensive collision checking between each disjoint sub-region of $3D$ space with the environment. Conversion from point clouds to binary voxel grids circumvents this issue by binning every point and assigning it an $(i, j, k)$ index. 

While there are many ways one could convert the ellipsoidal representation into a conservative occupancy grid, we propose the following method that is parallelizable and efficient, and show in \cref{sec:results} that it is not too conservative. Without loss of generality, we assume that the robot is a sphere, which can be done by applying the necessary rotation and stretching for all ellipsoids such that the robot ellipsoid is a sphere. For every ellipsoid, we calculate its axis-aligned bounding box. 

The Minkowski sum of the ellipsoid with a sphere does not present an ellipsoid. However, at the extremal points which represent intersections of the bounding box with the ellipsoid, the normal of the ellipsoid is in the principal directions. The bounding box is defined as the following
\begin{equation}
    \label{eq:min-bounding-box}
    \mathcal{B}_j = [\mu^i_j - \sqrt{\Sigma_{ii}},\ \mu^i_j + \sqrt{\Sigma_{ii}}]_{i=1}^3,
\end{equation}
where $\mu^i_j$ is the $i-$th element in the mean for ellipsoid $j$, and $\Sigma_{ii}$ represents the $i$-th diagonal term of the associated covariance.

For those $\mathcal{B}_j$ whose side lengths are not within the resolution of each grid cell $v_{x, y, z}$, we subdivide them by their largest side length relative to $v_{x, y, z}$. We iterate on this process until all subdivisions of $\mathcal{B}_j$ are smaller than $v_{x, y, z}$. At this point, we can calculate all 8 vertices for every subdivision of $\mathcal{B}_j$ and bin them similar to the point cloud case. This procedure can leverage batch operations on GPU and ensures that we construct an over-approximation of the collision geometry.

To account for the extent of the robot body, we convert the robot sphere into a kernel (similar to \cite{chen2023catnips}) and perform a MaxPool3D operation. The resultant grid represents where the robot can be centered and be safe or unsafe. More sophisticated subdivision routines may be used to reduce the conservativism of the grid. Once the final grid is constructed, we run Dijkstra to find the seed path represented as an ordered set of connected line segments $\mathcal{L} = \{ \ell_i \}_{i=1}^L$. 

\subsubsection{Collision Set}
Along the seed path, rather than checking collisions between the robot and \emph{every} ellipsoid in the scene, we would like to quickly find a subset of these primitives in the local vicinity of the robot to check against for efficiency reasons. In fact, \cref{prop:ellipsoidtest2} or \cref{cor:gs_intersection_test} can directly be used to define a ball or ellipsoid collision set centered around the seed path, but may contain unnecessary information at the cost of additional compute.  

Instead, following the paradigm of \cite{liu-safe-flight-corridors}, we can rapidly define a bounding box oriented along $\ell_i$ and pinpoint ellipsoids that live within it without incurring the additional cost of reasoning about the linear motion of the robot body. We define a radius $r_s = \frac{v_{\rm max}^2}{2 a_{\rm max}}$, which is the maximum stopping distance (dependent on the maximum velocity and acceleration), such that the facets of the box are no less than $r_s + \kappa$ away from $\ell_i$. This bounding box will be denoted as $A_i^{bb} x\leq B_i^{bb}$.

To check all ellipsoids that are at least partially contained within the box, we check for the minimum signed distance between each hyperplane $\min_{x \in \mathcal{E}_j} a_i^{bb} x \leq b_i^{bb}$ with every ellipsoid $\mathcal{E}_j$ in the scene. Ellipsoids that have negative signed distance for every hyperplane in the box will be at least partially contained. To perform this check, the plane and ellipsoid undergo an affine transformation to produce a new plane and an origin-centered sphere. The signed distance of the new plane from the origin must be less than 1, namely

\begin{equation}
    \label{eq:ellipsoid-plane-check}
    \frac{a_i^{bb} \mu_j - b_i^{bb}}{||a_i^{bb} R_j S_j||_2} \leq 1.
\end{equation}

\subsubsection{Polytope Generation}
The creation of polytopes around the line segment $\ell_i$ can be done through \cref{corr:supp-hyperplane-line} and appending these constraints to the bounding box constraints $a_i^{bb} x \leq b_i^{bb} - \kappa \|a_i^{bb}\|_2$. Note that if we were to create a halfspace for every ellipsoid in the constraint set, we would overly constrain the free space, leading to a smaller-than-necessary polytope. This phenomenon arises from the fact that, given an existing set of halfspaces, ellipsoids that are outside of the set can still contribute non-redundant halfspaces to the existing set. Moreover, having more halfspaces than necessary in the polytope representation can significantly slow down the proceeding spline optimization. 

Therefore, we adopt a greedy algorithm like \cite{liu-safe-flight-corridors}. Every time we form a new halfspace, we use \cref{eq:ellipsoid-plane-check} to eliminate from our collision set all ellipsoids that violate this halfspace. Of the remaining ellipsoids, we create a new halfspace for the one that had the smallest $K(s^*)$. We iterate this process until no ellipsoids remain in the collision set. 

Due to manageability, we can further reduce the complexity of our corridor representation by retrieving a smaller number $P$ of polytopes than line segments $L$. For the current part of the seed path, we construct the minimal collision set and the polytope $(A_p, b_p)$. Then, we check subsequent line segments, represented as the endpoints, with the current polytope. The first instance where the line segment is not fully contained in $(A_p, b_p)$, we construct a new minimal collision set and polytope and repeat the process until the end of the seed path. Keeping more polytopes enables smoother paths (e.g. less opportunity for pinch points) at the expense of higher computation in the spline optimization phase. 

\subsubsection{Spline Optimization}
Given the safe flight corridor represented as $P$ polytopes and initial and final configurations $(x_0, x_f)$, we compute a set of $P$ B\'ezier curves (parametrized by $M+1$ control points $c^m_p$ and Bernstein basis $\beta^m(t)$\footnote{For notational simplicitiy, we refer to the variable as both the conventional basis and its time derivatives up to some specified order $D$.} with progress $t \in [0, 1]$) representing the trajectory of the robot using the path-length minimization problem:
\begin{subequations} \label{eq:traj_computation}
    \begin{align}
    \min_{c^m_p, x_p(t)} &\sum_{p=1}^P \sum_{m = 0}^{M-1} ||c^p_{m+1} - c^p_m||_2^2 \\
    \subj \nonumber\\
    \text{Safety:} \;&A_p c_p^m \leq B_p, \; p=1,..., P; m=0,...,M\\
    \text{Configuration:}\; & x_0(0)= x_0\\
    &x_P(1) = x_f \nonumber\\
    \text{Continuity:}\;&x_p(1) = x_{p+1}(0), \;  p=1,...,P-1\\
    \text{B\'ezier curve:}\;&x_p(t) = \sum_{m=0}^M \beta^m(t) c^m_p, p=1,...,P-1\\
    \text{Dynamics:}\;&x_p(t+1) = f(x_p(t), u). \label{eq:dynamics}
    \end{align}
\end{subequations}
Without the dynamics constraints \eqref{eq:dynamics}, the optimization problem reduces to a quadratic program that can be solved in real-time, producing a trajectory that can be tracked by differentially-flat robots. The quadratic program is solved natively using Clarabel \cite{Clarabel_2024}. 

Due to the convex hull property of B\'ezier curves, constraining the control points to lie in the polytopes ensures that all points along the curves will lie in the corridor and hence guarantees safety in the continuum. Additionally, Splat-Plan is sound and complete, summarized in the following corollary.

\begin{corollary}
    \label{corr:completeness}
Given that there exists a smooth, safe path in an arbitrarily complex GSplat environment, Splat-Plan (\cref{alg:splat_plan}) will always return a feasible, safe, and smooth path as the limit of the voxel occupancy grid used for path seeding goes to $0$, the number of steps in $\cref{alg:bisection-search}$ goes to $\infty$, and the number of control points parametrizing the spline also goes to $\infty$.
\end{corollary}

\begin{proof}
As the occupancy grid resolution grows increasingly small, the unoccupied grid converges toward the true collision-free space of the scene. In the limit, Dijkstra will find an initial safe, feasible path toward the goal. Next, the constructed set of polytopes forming the corridor will (1) always be safe, (2) connected, and (3) contain the initial path assuming a sufficient number of iterations of \cref{alg:bisection-search} is performed due to \cref{corr:convergence_of_bisection} and manageability. Due to the Stone-Weierstrass theorem \cite{de1959stone}, given an arbitrary smooth curve, a B\'ezier curve of sufficient degree can exactly recover it. Under mild conditions, a smooth curve exists within the corridor, and given a sufficiently expressive parametrization and enough time, the quadratic program will find a solution.

Certainly, without these conditions, Dijkstra could fail to find a path for finite resolutions. Similarly, \cref{alg:bisection-search} and \cref{corr:supp-hyperplane-line} could return a conservative estimate that does not contain the line segment if given an insufficient number of iterations. Finally, \cref{eq:traj_computation} could return infeasibility if the degree of the B\'ezier curve is not expressive enough. Failure of these three components will lead Splat-Plan to not return a solution. However, in our experimental results, we find that this is not the case in practical settings. 
\end{proof}

\smallskip
\noindent \textbf{Querying Waypoints.}
For simplicity of notation, we will refer to the output of Splat-Plan as $X(T)$, which takes in metric time, finds the associated spline $X_p$, and queries the spline for its position and derivatives at the local spline time $\Delta_p T = T - T_{p, start}$. In our hardware demonstrations, we un-normalize the B\'ezier curves in order to approximately achieve the desired $v_{\rm max}$. Specifically, by knowing the associated subset of the seed path (and the total length $L_p$) for each polytope, we can enforce in metric time the duration of the individual splines $\Delta T_{p, 0} = \frac{L_p}{v_{max}}$. Our local B\'ezier curve is re-mapped using the formula $X_p(T) = x_p(\frac{\Delta_p T}{\Delta T_{p, 0}})$. \cref{eq:traj_computation} can again be used to constrain $X_p$ to lie within the safety corridor and enforce continuity in metric time. The entire Splat-Plan algorithm is visualized in \cref{fig:splatplan_explainer}.

\begin{algorithm2e}[th]
    \caption{Splat-Plan}
    \label{alg:splat_plan}
    \KwIn{$x_0$, $x_f$, grid resolution, lower, and upper bounds $(d, u_\ell, u_h)$, $v_{max}$, $a_{max}$, $\kappa$, $\mathcal{G},$ num. iters. $k$\;}
    \KwOut{B\'ezier spline $x(t)$\;}
    \tcp{Create the voxel grid}
    $V \gets \text{GSplatVoxelGrid}(d, u_\ell, u_h, \kappa, \mathcal{G})$\;
    \tcp{Path Seeding: creates safe but non-smooth path}
    $\{\ell\}_{i=1}^L \gets V(x_0, x_f)$\;
    \tcp{Iterate through the line segments}
    \For{$i\gets 1$ \KwTo $L$}{
        \tcp{Skip to next line segment if contained}
        \uIf{$\text{IsInPolytope}(\ell_i, (A_{curr}, B_{curr}))$}{
            continue\;
        }
    \tcp{Create collision set}
    $\mathcal{G}_i, (A_i^{bb}, B_i^{bb}) \gets \text{GetCollisionSet}(\ell_i, \kappa, \mathcal{G}, v_{max}, a_{max})$\;
    \tcp{Initialize polytope}
    $A_i, B_i \gets A_i^{bb}, B_i^{bb}$\;
    \tcp{Create polytope}
    \While{$|\mathcal{G}_i| > 0$}{
    $\{K_j\} \gets \text{\cref{alg:bisection-search}}(\ell_i, \mathcal{G}_i, k)$\;
    $K \gets \min(\{K_j\})$\;
    \tcp{Create halfspace}
    $(A, B) \gets \text{\cref{corr:supp-hyperplane-line}}(\ell_i, K)$\;
    \tcp{Add halfspace to polytope}
    $A_i \gets \text{append}(A_i, A), B_i \gets \text{append}(B_i, B)$\;
    \tcp{Reject redundant ellipsoids}
    $\mathcal{G}_i \gets \text{\cref{eq:ellipsoid-plane-check}}\ (\mathcal{G}_i, (A, B))$\;
    }
    \tcp{Set current polytope}
    $A_{curr} \gets A_i, B_{curr} \gets B_i$\;
    }
    \tcp{Optimize B\'ezier splines}
    $x(t) \gets \text{Optimize}(x_0, x_f, \{(A_p, B_p) \}_{p=1}^P)$
\end{algorithm2e}

{
\section{Monocular Pose Estimation}
\label{sec:estimation}
In this section, we present our pose estimation module, Splat-Loc, for localizing a robot in a GSplat representation of its environment. This is essential to the overall functionality of the SplatNav pipeline as the safety guarantees of Splat-Plan only hold if the robot is able to consistently and accurately estimate its pose in the GSplat map. 
Splat-Loc only requires a monocular RGB camera, which enables it to work on a broad range of hardware platforms, including those beyond robots (such as mobile phones). Furthermore, Splat-Loc can be used either as a stand-alone pose estimation system or in conjunction with an independent pose estimation system (onboard VIO, external motion capture, etc). 

\smallskip
\noindent \textbf{Problem Formulation.}
Formally, we wish to estimate the pose of a robot at a particular time $\hat{T}_t \in \SE{(3)}$ given a color image $I_t \in \mbb{R}^{H \times W \times 3}$. The true camera pose $T_t$ is unknown. A pose in \SE${(3)}$ is parameterized by a rotation matrix ${R \in \SO(3)}$ and a translation vector ${\tau \in \mbb{R}^{3}}$
\begin{equation}
    \label{eq:exponential_map}
    T_t =\left(
        \begin{array}{c | c}
            R_t & \tau_t \\
            \hline
            \mb{0}_{1\times3} & 0
        \end{array}
    \right).
\end{equation}

In the case that the navigating robot has an independent pose estimation system, we would like to use those pose estimates as initializations for Splat-Loc's optimization procedures, and also correct these poses using the estimates from Splat-Loc.
We assume knowledge of the camera's calibration including the intrinsic matrix and distortion coefficients for projective geometry. These are easily computable, and are often available from the camera manufacturer.



\smallskip
\noindent \textbf{Lightweight Monocular Pose Estimator.}
At its core, Splat-Loc uses the fast rendering capabilities of GSplats and standard tools from camera tracking to formulate Perspective-n-Point (PnP) problems, which can be reliably solved using off-the-shelf optimizers, and produces accurate estimates of the robot pose. 
As input for the pose estimation procedure, we have the color image and a coarse initial guess for the pose estimate, $\hat{T}_{t,0}$. This guess can either come from an independent localization module (e.g. VIO) or can simply be the previous time step's estimate. We begin by rendering an RGB image using the GSplat map with the camera pose set to the initial guess and simultaneously generate a local point-cloud within the camera's view, effectively using the GSplat as a monocular depth estimator.

Next, a local feature extractor is used to compute visual features (keypoints and descriptors) in both the camera image and the rendered image. Each keypoint has an associated pixel coordinate $(u, v) \in \mbb{R}^2$, and let $m$ and $n$ respectively be the number of keypoints in the camera and rendered images. A feature matcher is used to determine correspondences between the visual features in the camera image and the rendered image. Let $\ell \leq \min\{m, n\}$ be the number of successfully matched features. In our experiments we found that the feature extractor SuperPoint \cite{detone2018superpoint} used in conjunction with the transformer-based LightGlue \cite{lindenberger2306lightglue} feature matcher had the best performance (see \cref{sec:results} for more details).

Using the rendered depth image and the camera intrinsics matrix, the keypoints from the rendered color image can be projected into the 3D to produce a point cloud. Let $\hat{p}_j \in \mbb{R}^3$ be the position of the $j$th projected keypoint where $j \in {1, \dots, n}$. Finally, we seek to minimize the following reprojection error in order to find the relative pose transform that transforms $\hat{T}_{t,0}$ to $\hat{T}_t$
\begin{equation}
    \label{eq:pnp_problem}
    \bar{T}_t = \argmin_{T \in \SE(3)} \sum_{k=1}^\ell ||\rho_{k} - K T \hat{p}_k||_2,
    ,
\end{equation}
where ${\rho_{k} = [u_k, v_k, 1]^{\top}}$ and subsequently recover our estimated pose $\hat{T}_t = \bar{T}_t \hat{T}_{t, 0}$. \cref{eq:pnp_problem} is the Perspective-n-Point problem, and is a nonlinear least-squares optimization problem that we solve using the Levenberg-Marquardt algorithm. 
In practice, we use Random Sample Consensus (RANSAC) to remove outliers from the set of matched features which results in more robust solutions of \cref{eq:pnp_problem}. We illustrate this procedure in \Cref{fig:pipeline}. In \cref{sec:results}, we highlight the accuracy of incremental estimation in real-world experiments while a drone navigates a cluttered environment.

\smallskip
\noindent \textbf{Global Initialization.}
\label{sec:global_initialization}
The above pose estimation procedure requires common overlap between $I_t$ and $\hat{I}_t$, necessitating a reasonably accurate initial estimate of the robot's pose $\hat{T}_{t, 0}$, which may not be available in many practical settings. When a good initial guess of the robot's pose is unavailable, we execute a global pose estimation procedure. Note that this only needs to be performed once, and then  subsequent pose estimate steps can be performed using the solution from the previous iteration.

One approach requires a monocular depth estimator, e.g., \cite{bhat2023zoedepth, yang2024depth}, to augment the RGB image obtained by the robot with depth information, which is used to generate a point cloud (in the camera frame). Another is to randomly sample $\SE{(3)}$ for pose initializations and return the pose estimate from the P$n$P run that has the lowest reprojection error. However, we instead generate a point cloud of the scene from the GSplat means $\{\mu_j\}_{j = 1}^{N}$, enabling the formulation of a point-cloud registration problem:
\begin{equation}
    \label{eq:point_to_point_registration}
    \hat{T}_{0} = \argmin_{T \in \SE(3)} \sum_{(p, q) \in \mcal{C}} \norm{p - Tq}_{2}^{2},
\end{equation}
where ${C}$ denotes the set of correspondences, associating the point $p$ in the map cloud to a point $q$ in the point cloud from the camera. If we are given a known set of correspondences, we can compute the optimal solution of \eqref{eq:point_to_point_registration} using Umeyama's method \cite{umeyama1991least}.

In practice, we do not have prior knowledge of the set of correspondences $\mcal{C}$ between the two point clouds. To address this challenge, we apply standard techniques in feature-based global point-cloud registration. We begin by computing $33$-dimensional Fast Point Feature Histograms (FPFH) descriptors \cite{rusu2009fast} for each point in the point-cloud, encoding the local geometric properties of each point. Prior work has shown that visual attributes can play an important role in improving the convergence speed of point-cloud registration algorithms \cite{park2017colored}, something that FPFH does not do. To solve this, we augment the FPFH feature descriptor of a given point with its RGB color. We then identify putative sets of correspondences using a nearest-neighbor query based on the augmented FPFH descriptors, before running RANSAC to iteratively identify and remove outliers in $\mathcal{C}$. The RANSAC convergence criterion is based on the distance between the aligned point clouds and the length of a pair of edges defined by the set of correspondences.

\smallskip
\noindent \textbf{Non-invasive Pose Correction.}
While fusing Splat-Loc poses with existing pose estimates like VIO is beyond the scope of this work, we will address challenges that arises when using Splat-Plan to plan high-level plans in a GSplat while using existing pose estimates to stabilize (i.e., for control). Fundamentally, discrepancies between the one in which the GSplat is trained in $\mathcal{T}_{\text{gs}}$ and the running coordinate frame of the existing localization module $\mathcal{T}_{(\text{control}, t)}$ can vary with time, either due to noise or drift. Yet, poses from Splat-Loc are inherently tied to the GSplat coordinate frame, leading to potentially more informative state estimates of whether the robot is in collision or not. In turn, these estimates can be passed into Splat-Plan to create safer trajectories if necessary, as depicted in \cref{fig:pipeline}. However, the trajectory that Splat-Plan returns again lives in $\mathcal{T}_{\text{gs}}$ and not necessarily the running coordinate frame of the existing localization, which is crucially used for control. To overcome this mismatch, we necessarily need to transform the outputs of Splat-Plan into the control localization frame. Namely, there exists a transform $^{\text{control}, t} T_{\text{gs}}: \mathcal{T}_{\text{gs}} \to \mathcal{T}_{(\text{control}, t)}$ that maps poses in the GSplat frame to ones in the control localization frame. Therefore, the waypoints that we send to the robot are $^{(\text{control}, t)} T_{\text{gs}}(X(T))$, which is depicted in \cref{fig:pipeline} as the input to the robot.

\section{Experiments}
\label{sec:results}
We demonstrate the effectiveness of our navigation pipeline for GSplat maps, examining its performance in real-world scenes on hardware and in simulation. In addition, we perform ablative studies comparing our algorithms against existing methods. 

\subsection{Simulation Results}

\subsubsection{Test Environments}
We benchmark Splat-Plan and Splat-Loc independently on four different environments: \textbf{Stonehenge}, a fully-synthetic scene, and three real-world scenes \textbf{Statues}, \textbf{Flightroom}, and \textbf{Old Union}. For \textbf{Stonehenge}, we captured image-pose pairs by rendering the \textbf{Stonehenge} mesh in Blender. For the other scenes, we recorded a video from a mobile phone and processed the image frames through structure-from-motion (COLMAP \cite{schoenberger2016sfm}) to retrieve corresponding camera poses and intrinsics. 

\subsubsection{Splat-Loc Evaluations}
\label{sssec:pose_estimation_mapping_localization}
We compare Splat-Loc to existing pose estimation methods, including a baseline GS-Loc, based on the localization component of existing GSplat SLAM methods \cite{yan2023gs, yugay2023gaussian}. We leverage finite differences to estimate the gradient of the photometric loss function utilized in the pose estimator, which might not be particularly fast or robust, especially for larger errors in the initial pose estimate. While these methods optimize over the re-rendering loss composed of the photometric loss, and in some cases, depth and semantic-related loss terms, in our baseline, we optimize only over the photometric loss, since we assume the robot in these evaluations does not have an RGB-D camera for depth measurements. As a result, our baseline essentially matches the GSplat SLAM method in \cite{matsuki2023gaussian}. 
In addition, we compare our pose estimator to the Point-to-plane Iterative Closest Point (ICP) \cite{chen1992object} and Colored-ICP \cite{park2017colored} algorithms, assuming these point-cloud methods have privileged $3$D information that the incremental estimation of Splat-Loc does not have. 
Furthermore, we examine two variants of our pose estimator: Splat-Loc-Glue, which utilizes LightGlue for feature matching; and Splat-Loc-SIFT, which utilizes SIFT for feature matching. 

In each scene, we run 10 trials (of $100$ frames each) of each pose estimation algorithm. We evaluate the rotation error (R.E.) and translation error (T.E.) with respect to the ground-truth pose, the computation time (C.T.) per frame, and the overall success rate (S.R.). Here, success indicates the generation of a solution regardless of its quality.
The performance of pose estimation algorithms often depends on the error associated with the initial estimate of the pose. As such, we test our system across a range of different errors in the initial estimate of the pose. In this study, we assume an initial estimate of the pose is available. We generate the initial estimate by taking the ground truth pose then applying a rotation $\delta_R$ about a random axis and the translation $\delta_t$ in a random direction.

\ifbool{enable_concise_mode}
{}
{
We note that many currently-available implementations of differentiable rasterizers for Gaussian Splatting do not provide the gradient of the re-rendering loss with respect to the camera pose \cite{kerbl3Dgaussians, Ye_gsplat}, required for gradient-based pose estimation. In addition, Gaussian Splatting does not support automatic differentiation, in its original form. As a result, we leverage finite differences to estimate the gradient of the photometric loss function utilized in the pose estimator. 

We note that our implementation might not be particularly fast or robust, especially for larger errors in the initial pose estimate, given the numerical approach utilized in estimating the gradients. 
}

We provide the summary statistics of the error in the pose estimates computed by each algorithm, in addition to the computation time on a trial with $100$ frames in the \textbf{Statues} scene in Table~\ref{tab:pose_estimation_results_statues_smaller_error_inerf}. We note that all methods had a perfect success rate in this problem. The GS-Loc algorithm achieves the lowest accuracy and requires the greatest computation time, unlike Colored-ICP, Splat-Loc-SIFT, and Splat-Loc-Glue, which achieve much-higher accuracy with a rotation error less than a degree and a translation error less than $15$cm. GS-Loc requires a computation time of about \unit[36.15]{s} per frame, which is about two orders of magnitude slower than the next-slowest method ICP, which requires a computation time of about \unit[110]{ms}. Colored-ICP, \mbox{Splat-Loc-SIFT}, and \mbox{Splat-Loc-Glue} require less than \unit[100]{ms} of computation time. Compared to all methods, \mbox{Splat-Loc-Glue} yields pose estimates with the lowest mean rotation and translation error, less than $0.06^\circ$ and \unit[4]{mm}, respectively, and achieves the fastest mean computation time, less than \unit[42]{ms}. The computation time of Splat-Loc may be about a standard deviation greater during the first call, which may be due to the time spent loading the models and initializing the GPU kernels. 

\begin{table}[h!]
	\centering
	\caption{Comparison of baseline pose estimation algorithms in simulation in the \textbf{Statues} scene with ${\delta_R = 20^\circ}$ and ${\delta_t = \unit[0.1]{m}}$.
        }
	\label{tab:pose_estimation_results_statues_smaller_error_inerf}
	\begin{adjustbox}{width=\linewidth}
		{\begin{tabular}{l c c c c}
				\toprule
				Algorithm & R.E. (deg.) & T.E. (cm) & C.T. (msec.) & S.R. ($\%$) \\
				\midrule
				ICP \cite{chen1992object} & $73.1 \pm 45.9$ & $129 \pm 75$ & $107 \pm 19.2$ & $100$ \\
				Colored-ICP \cite{park2017colored} & $0.83 \pm 0.37$ & $1.31 \pm 0.60$ & $43.3 \pm 9.70$ & 100 \\
				\mbox{Splat-Loc-SIFT} (ours) & $0.09 \pm 0.06$ & $0.56 \pm 0.75$ & $63.3 \pm 2.39$ & 100 \\
				\mbox{Splat-Loc-Glue} (ours) & $\bm{0.05 \pm 0.03}$ & $\bm{0.33 \pm 0.27}$ & $\bm{41.2 \pm 73.2}$ & 100 \\
				GS-Loc \cite{yan2023gs, yugay2023gaussian, matsuki2023gaussian} & $122 \pm 33.8$ & $245 \pm 91.2$ & $36200 \pm 5440$ & $100$ \\
				\bottomrule
		\end{tabular}}
	\end{adjustbox}
\end{table}

Lastly, we examine the performance of the pose estimation algorithms in problems with a larger error in the initial estimate of the pose, with $\delta_R = 30^\circ$ and $\delta_t = \unit[0.5]{m}$ in the synthetic \textbf{Stonehenge} scene. We present the performance of each algorithm on each metric in Table~\ref{tab:pose_estimation_results_stonehenge_larger_error}, where we note that ICP and Colored-ICP do not provide accurate estimates of the robot's pose. Moreover, the pose estimation errors achieved by ICP and Colored-ICP have a significant variance. In contrast, \mbox{Splat-Loc-SIFT} and \mbox{Splat-Loc-Glue} yield pose estimates of high accuracy with average rotation and translation errors less than $0.5$ deg. and $5$mm, respectively. However, \mbox{Splat-Loc-SIFT} achieves a lower success rate, compared to \mbox{Splat-Loc-Glue}, which achieves a perfect success rate. 

\begin{table}[th]
	\centering
	\caption{Comparison of baseline pose estimation algorithms in the \textbf{Stonehenge} scene with $\delta_R = 30^\circ$ and $\delta_t = \unit[0.5]{m}$.
        }
	\label{tab:pose_estimation_results_stonehenge_larger_error}
	\begin{adjustbox}{width=\linewidth}
		{\begin{tabular}{l c c c c}
				\toprule
				Algorithm & R.E. (deg.) & T.E. (cm) & C.T. (msec.) & S.R. ($\%$) \\
				\midrule
				ICP \cite{chen1992object} & $131 \pm 22.6$ & $370 \pm 554$ & $122 \pm 153$ & $100$ \\
				Colored-ICP \cite{park2017colored} & $94.9 \pm 51.3$ & $57.4 \pm 28.8$ & $488 \pm 104$ & $20$ \\
				\mbox{Splat-Loc-SIFT} (ours) & $\bm{0.217 \pm 0.0369}$ & $0.334 \pm 0.0563$ & $139 \pm 3.15$ & $70$ \\
				\mbox{Splat-Loc-Glue} (ours) & $0.220 \pm 0.203$ & $\bm{0.315 \pm 0.210}$ & $\bm{45.1 \pm 0.611}$ & $100$ \\
				\bottomrule
		\end{tabular}}
	\end{adjustbox}
\end{table}

\subsubsection{Splat-Plan Evaluations}
Splat-Plan is benchmarked against three different methods: a point-cloud planner \cite{liu-safe-flight-corridors}, a sampling-based planner (RRT* using \cref{prop:ellipsoidtest2}), and a NeRF-based planner (NeRF-Nav \cite{adamkiewicz2022vision}). Furthermore, we perform ablations against variations of the point-cloud planner in order to expose flaws when planning against point clouds compared to the full scene geometry. For each simulation scene, we train a dense and sparse GSplat, totaling 8 scenes. In every scene, we run 100 start and goal locations distributed in a circle around the boundary of the scene.

In the simulated tests, we represent the robot using balls of various sizes in order to generate interesting trajectories due to the fact that the simulated scenes are not trained in metric scale.\footnote{Nerfstudio adopts the NeRF conventions in scaling the scene to fit within the confines of a two-unit-length cube centered at the origin, with the poses of the camera residing within a $[-1, 1]^3$-bounding box. We disable this feature for the hardware \textbf{Maze} scene.} . Additional parameters, such as the number of Gaussians, can be found in \cref{tab:params}.

\begin{table}[h]
  \centering
  \caption{Parameters across scenes for experiments. Number of Gaussians is reported for both dense and sparse variants of the same scene.}
  \label{tab:params}
  \begin{tabular}{|c||c|c|c|c|c|}
\hline
 & Radius & $V_{max}$ & $A_{max}$ & Resolution & N. Gauss (K)\\ \hline\hline
    Stonehenge & 0.01 & 0.1 & 0.1 & $150^3$ & $116/12$\\ \hline
    Statues & 0.03 & 0.1 & 0.1 & $100^3$ & $201/18$\\ \hline
    Flight & 0.02 & 0.1 & 0.1 & $100^3$ & $281/4$\\ \hline
    Old Union & 0.01 & 0.1 & 0.1 & $100^3$ & $525/87$\\ \hline
    Maze & 0.25 & 0.5 & 1.0 & $80^3$ & $100$\\ \hline
    Maze (fast) & 0.25 & 1.5 & 1.0 & $80^3$ & $100$\\ \hline
  \end{tabular} 
\end{table}

While point cloud-based planners are ubiquitously used, they can sometimes fall short when the scene geometry is not dense or if the scene is very cluttered. To this end, we developed four variants of the Safe Flight Corridor (SFC) \cite{liu-safe-flight-corridors}. SFC-1 ingests the GSplat means as a point cloud, runs Dijkstra to retrieve a feasible initial path seed, creates collision sets with respect to the point cloud, synthesizes a polytope corridor that marginally intersects with the point cloud, and finally deflates the polytopes by the robot radius. These polytopes are fed to the same spline optimizer \eqref{eq:traj_computation} that Splat-Plan uses. \mbox{SFC-2} executes the same pipeline as SFC-1, but the point cloud representation is sampled from the surface of the ellipsoids. We sample 20 points from each ellipsoid in the scene to simulate a typical amount of points a Lidar or depth image would produce (approximately 2-5 million points). SFC-3 uses the \mbox{Splat-Plan} occupancy grid to retrieve a feasible path seed, while the means are still used to create polytopes. Finally, SFC-4 uses the Splat-Plan occupancy grid, synthesizes polytopes using the means, but deflates the polytope by the robot radius and the maximum eigenvalue of the ellipsoid whose mean was used to create a particular halfspace in the polytope. These variants are all potential solutions to apply SFC to GSplat environments. We summarize the tradeoffs of all methods in \cref{tab:tradeoffs}.

\begin{table}[h]
  \centering
  \caption{Splat-Plan strikes a favorable tradeoff compared to existing methods in terms of safety, non-conservativeness (NC), smoothness, solution feasibility, and real-time execution.}
  \label{tab:tradeoffs}
  \begin{tabular}{|c||c|c|c|c|c|c|}
    \hline
\cline{2-7}
     & Safe & NC & Smooth & Feasible & RT & Env. \\ \hline\hline
    NeRF-Nav & \texttimes & N/A & \texttimes & \checkmark & \texttimes & NeRF  \\ \hline
    RRT* & \texttimes & \texttimes & \texttimes &\checkmark & \texttimes & GS  \\ \hline
    SFC-1 & \texttimes & \checkmark & \checkmark & \texttimes &\checkmark & GS  \\ \hline
    SFC-2 & \texttimes & \checkmark & \checkmark & \texttimes &\checkmark & GS  \\ \hline
    SFC-3 & \texttimes & \checkmark & \checkmark & \texttimes &\checkmark & GS \\ \hline
    SFC-4 & \texttimes & \texttimes & \checkmark & \texttimes &\checkmark & GS  \\ \hline \hline 
\rowcolor[gray]{0.8} \textbf{Splat-Plan} & \checkmark & \checkmark & \checkmark &
 \checkmark & \checkmark & GS \\ \hline
  \end{tabular}
\end{table}

Visually, the paths generated by \mbox{Splat-Plan} are smooth, safe, and non-conservative (Fig. \ref{fig:planning_qualitative}). This fact is validated in Fig. \ref{fig:planning_quantitative}, where Splat-Plan's trajectories in blue are safe (minimum distances greater than 0 with respect to the GSplat collision geometry). Unfortunately, because many of these scenes were captured in the real-world, no ground-truth mesh exists. Moreover, we inspect the point cloud and mesh created by COLMAP and notice poor overall reconstruction of the collision geometry. Therefore, we elected to use the GSplat ellipsoidal geometry in place of the ground-truth geometry due to its high-quality approximation.

Notice that these trajectories are non-conservative compared to the SFC methods (low path lengths and high polytope volume in \cref{fig:planning_quantitative,fig:sfc_ablations}). More importantly, we see that Splat-Plan never fails to return a trajectory, highlighted by the 0 failure rate. All other methods have failures, other than NeRF-Nav by virtue of it being an end-to-end optimization method. Finally, Splat-Plan has comparable execution times to SFC. Note that as SFC does not use GPU, we rewrote the codebase in Pytorch to yield comparable times to Splat-Plan. 

\begin{figure*}[t]
    \centering
    \includegraphics[width=\textwidth, trim={0em, 0em, 0em, 0}, clip]{
    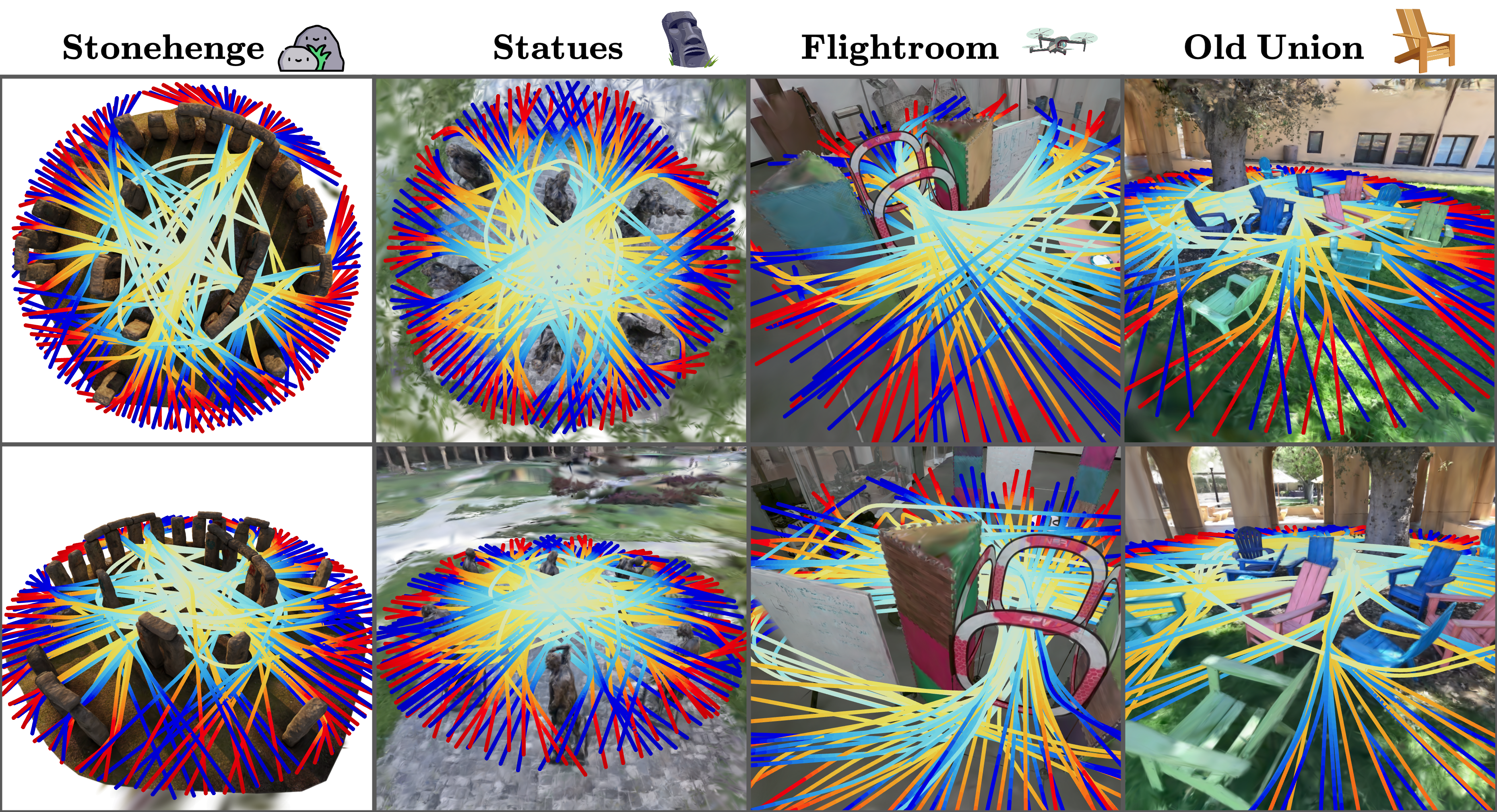
    }
    \caption{Qualitative results of 100 safe trajectories using Splat-Plan with start (blue) and goal (red) states spread over a circle. We see that the trajectories are safe, but not conservative.}
    \label{fig:planning_qualitative}
\end{figure*}

\begin{figure}[t]
    \centering
    \includegraphics[width=\columnwidth, trim={0em, 0em, 0em, 0}, clip]{
    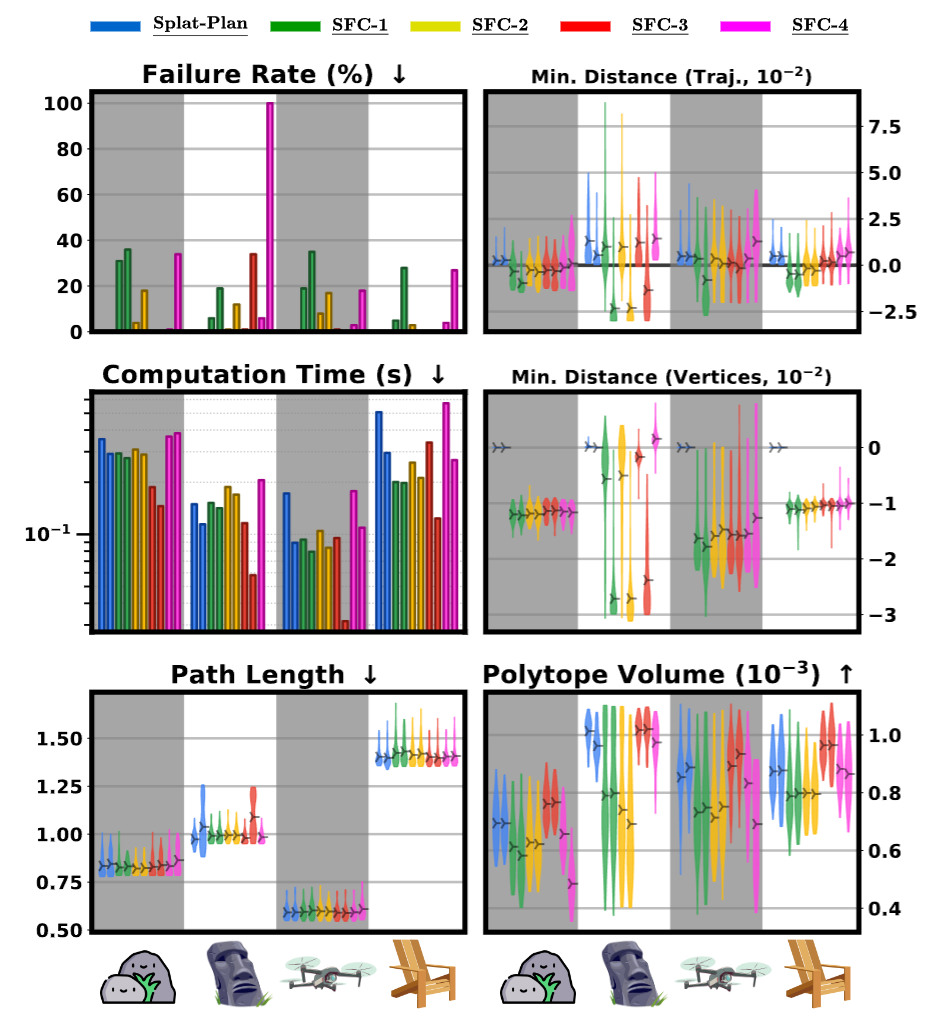
    }
    \caption{A comparison of trajectories generated by Splat-Plan and four variants of SFC \cite{liu-safe-flight-corridors} for the same 100 start and end locations for each scene (Stonehenge, Statues, Flightroom, and Old Union). Bars on the left represent trajectories planned using the dense scene, while the sparse scene is on the right. We evaluate the safety (Distance to GSplat and failure rate), conservativeness (path length and polytope volume), and computation time. We also show the distance of the polytope vertices to the GSplat to illustrate that Splat-Plan is sound. Violin plots showcase the spread of the minimum distance, path length, and polytope volume across all 100 trajectories, with markings indicating the mean. Splat-Plan displays competitive non-conservativeness and computation time, while exhibiting superior safety and success rates.}
    \label{fig:sfc_ablations}
\end{figure}

\begin{figure}[th]
    \centering
    \includegraphics[width=\columnwidth, trim={0em, 0em, 0em, 0}, clip]{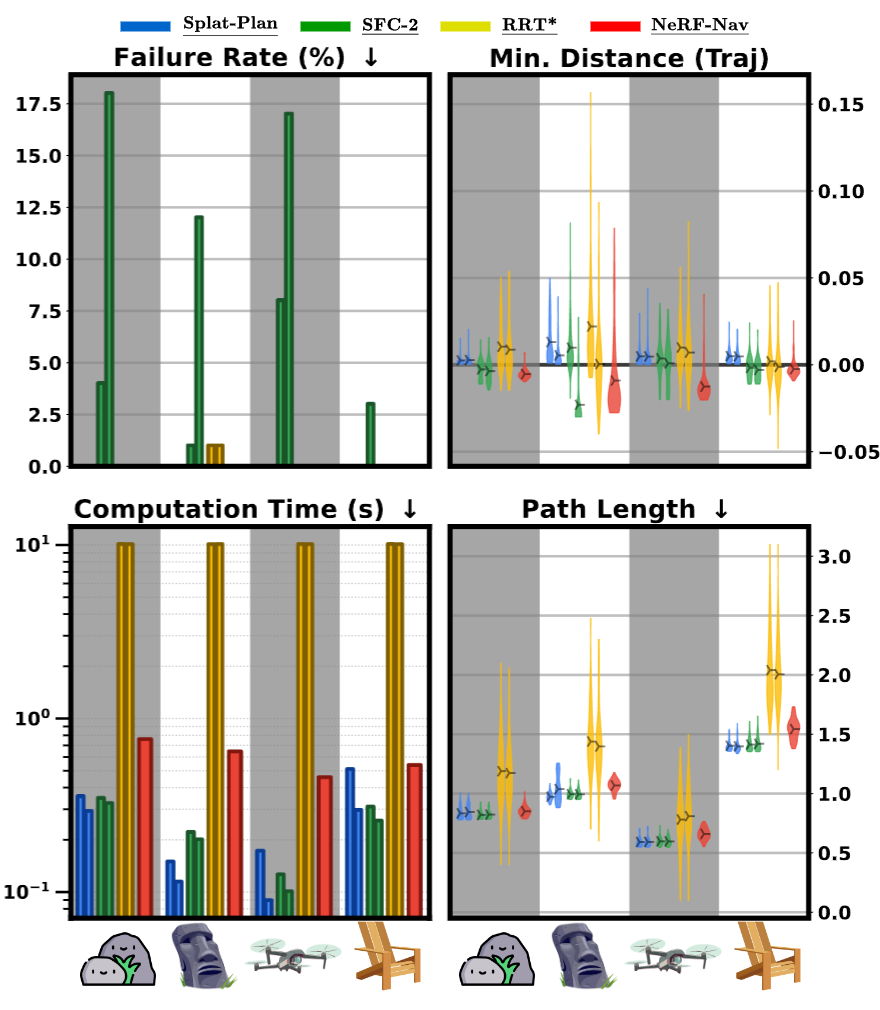}
    \caption{A comparison of trajectories generated by Splat-Plan, SFC \cite{liu-safe-flight-corridors}, RRT*, and NeRF-Nav \cite{adamkiewicz2022vision}) for the same 100 start and end locations for each scene (Stonehenge, Statues, Flightroom, and Old Union). Bars on the left represent trajectories planned using the dense scene, while the sparse scene is on the right. We evaluate the safety (Distance to GSplat and failure rate), conservativeness (path length), and computation time. Violin plots showcase the spread of the minimum distance and path length across all 100 trajectories, with markings indicating the mean. Splat-Plan is safer than competing methods, relatively fast, non-conservative, and never fails across all 8 problem settings.}
    \label{fig:planning_quantitative}
\end{figure}

Finally, in terms of memory, we observe that in the scene with the most Gaussians (\textbf{Old Union}), GPU memory usage hovered around 3.1 GB, with the GSplat itself requiring 1.6 GB and the binary occupancy grid, 1.5GB.

\subsection{Hardware Results}

\subsubsection{Test Environment}
We test Splat-Nav in the \textbf{Maze} scene using a drone. Images to train \textbf{Maze} were captured using the RGB camera onboard the drone. We utilize Nerfstudio \cite{nerfstudio} to train the Semantic GSplat, using its default parameters (which includes estimating the camera poses for each image frame from structure-from-motion via COLMAP \cite{schoenberger2016sfm}). 
%
In Figure~\ref{fig:maze_scene}, we show the true training images captured by the drone, the rendered RGB image from the GSplat at the same camera pose, and the semantic relevancy for the associated language query. First, we note that the rendered image is photorealistic, highlighting the remarkable visual quality of the trained Gaussian Splat. Second, the semantic relevancy spatially agrees with the expected location of the queried object, making the semantic field suitable for open-vocabulary goal querying.

\begin{figure}[th]
    \centering
    \includegraphics[width=\columnwidth]{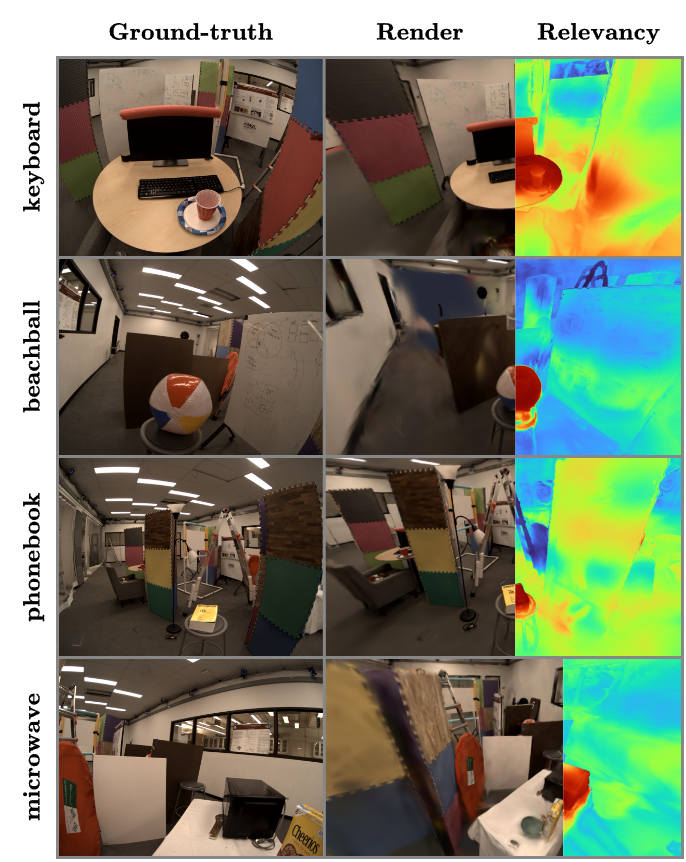}
    \caption{Natural language-specified goal locations in the \textbf{Maze} scene. The rendered RGB image from the GSplat demonstrate good reconstruction of the ground-truth. Additionally, the semantic relevancy spatially agrees with the provided language query. 
    }
    \label{fig:maze_scene}
\end{figure}

\subsubsection{Hardware}
We test our pipeline on the Modal AI development drone platform measuring $29$~cm x $20$~cm x $10$~cm (diagonal length of $36.6$~cm). In the hardware tests, we approximate the robot using a sphere with diameter $\unit[0.5]{m}$. Readers can find our test parameters in \cref{tab:params}. An OptiTrack motion capture system is solely used for evaluation purposes. Any other markers, such as ArUco tags, in the scene are purely cosmetic.

\subsubsection{Implementation}
We run the pose estimator and the planner ROS2 nodes on a desktop computer with an Nvidia RTX 4090 GPU and an Intel i9 13900K CPU, which communicates with the drone via WiFi. We emphasize that both modules are running asynchronously. At a frequency of about \unit[3]{Hz}, the drone transmits images from its cameras and associated VIO poses to the desktop computer. Splat-Loc ingests the VIO pose $\hat{T}_{t, 0}$ and the image $I_t$ to compute the pose estimate $\hat{T}_t$ of the drone body after applying rigid body transforms to transform the camera pose to the body frame. We run the estimator continuously, synchronized with the stream of images published from the drone via ROS2. 

The planning module ingests $\hat{T}_t$ as $x_0$ and computes a safe trajectory for the drone to follow toward the language-conditioned goal $x_f$. The re-plan node, which runs Splat-Plan based on $x_0$, updates the spline(s) $X(T)$ as frequently as possible. The waypoint node, which operates asynchronously from the re-plan node, measures the running time since the $X(T)$ was last updated and returns positions, velocities, acceleration, and jerk at this running time. The waypoint node runs at \unit[10]{Hz}. This architecture allows the drone to continue following a smooth spline even when Splat-Plan is still computing the next plan. However, we find that simply sending position waypoints can cause the drone to jerk when $X(T)$ is updated, as successive splines need not be close to one another. To rectify this issue, we forward integrate the waypoint velocities to get positions. Finally, these positions undergo $^{(\text{control}, t)} T_{\text{gs}}(X(T))$ before being sent to the drone. Additionally, we run a Kalman Filter to smooth $^{(\text{control}, t)} T_{\text{gs}}$, as both the Splat-Loc and VIO pose estimates can be somewhat noisy. 

\subsubsection{Goal Specification}
\label{sec:goal_specification}
In the \textbf{Maze}, we specify the goal locations for the drone via natural language, comprising of the following objects: a keyboard, beachball, phonebook, and microwave (depicted in Figure~\ref{fig:maze_scene}). We query the semantic Gaussian Splat for the location of these objects using the following text prompts: ``keyboard," ``beachball," ``phonebook" and ``microwave," corresponding to these objects, without negative prompts. These objects are placed in locations that require dynamic motions, such as hard turns and elevation maneuvers, to reach. Moreover, all tests begin at the same position at hover, and the objects are positioned relative to this position so that they are not immediately visible when the drone first begins flight.

\subsubsection{Control Schemes}
Our hardware tests consist of three different control schemes, coined Open-loop, Closed-loop VIO, and Splat-Loc. Open-loop tests do not not re-plan, and therefore does not use Splat-Loc estimates. One trajectory is created at the start $T = 0$, and the control node returns the corresponding waypoint at that point in time. No forward integration of the velocities is necessary since only one trajectory is ever created. Closed-loop VIO and Splat-Loc are re-planning control schemes where the re-plan node updates the trajectory $X(T)$ as frequently as possible. Closed-loop VIO uses the VIO estimate as $x_0$ and no additional transform is applied to the waypoint. Conversely, Splat-Loc uses the Splat-Loc pose estimate as $x_0$, and the smoothed $^{(\text{control}, t)} T_{\text{gs}}$ is applied to the Splat-Plan trajectories to transform them into the VIO control frame. 
Our hardware tests consist of all combinations of goal locations and control schemes. In addition, we run these combinations 10 times for statistical significance, yielding a total of 120 flights. 

\subsubsection{Splat-Loc Evaluations}
We validate the performance of Splat-Loc in 
hardware experiments in the \textbf{Maze} scene, showing that Splat-Loc achieves relatively the same level of accuracy as the onboard VIO in estimating the drone's pose, without requiring any special calibration or re-initialization procedures for frame alignment, which the onboard VIO requires. In Table~\ref{tab:pose_estimation_results_hardware_splat_loc_mocap}, we provide the rotation and translation errors of the Splat-Loc estimates, with the MOCAP poses as the ground-truth estimates. We note that Splat-Loc achieves rotation errors of about \unit[3]{deg} and translation errors of about \unit[4] {cm}, which is comparable to the accuracy of the VIO estimates, shown in Table~\ref{tab:pose_estimation_results_hardware_vio_mocap}. However, Splat-Loc failed in one of the closed-loop trials with the ``keyboard" goal location. As a result, the rotation and translation errors for this goal location is higher compared to the those of the other goal locations. The failure case is visualized in \Cref{fig:hardware_viz}, where the drone goes past the keyboard. We note that the failure likely occurred because the drone's camera was pointing towards an area of the scene which was not really covered in the video used in training the GSplat. We discuss strategies for addressing such failure cases in \Cref{sec:limitations_future_work}. In \Cref{fig:maze_scene}, we show the estimated trajectories of the drone using MOCAP, the onboard VIO, and Splat-Loc, demonstrating the effectiveness of Splat-Loc. Essentially, all the pose estimators achieve comparable estimation accuracy. However, unlike the MOCAP system, Splat-Loc does not require a specialized hardware system and is amenable to any monocular camera.
Moreover, Splat-Loc runs at about $25$ Hz on average, which is fast-enough for real-time operation. The bulk of the computation time is utilized in computing the feature matches and in solving the PnP problem, which requires about $10$ milliseconds. 


\begin{figure}[th]
    \centering
    \includegraphics[width=\columnwidth]{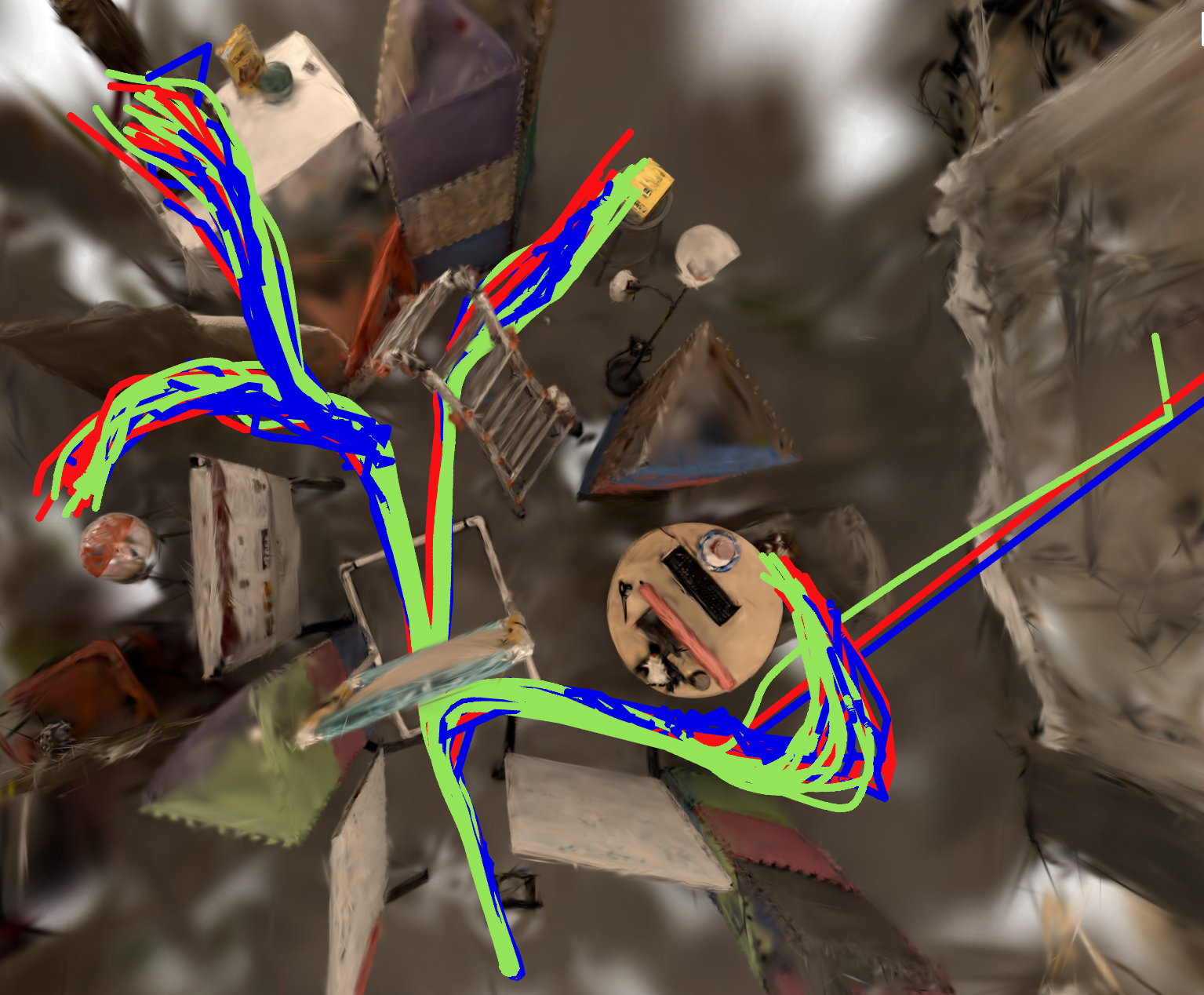}
    \caption{Pose estimates of Splat-Nav using motion capture (green), onboard VIO (red), and Splat-Loc (blue).  Splat-Loc gives comparable performance without requiring a manual frame alignment between the GSplat and a separate localization frame.}
    \label{fig:maze_scene}
\end{figure}



\begin{table}[th]
	\centering
	\caption{Rotation Error (R.E.), Translation Error (T.E.), and the Success Rate (S.R.) of Splat-Loc with MOCAP as the ground-truth reference.
        }
	\label{tab:pose_estimation_results_hardware_splat_loc_mocap}
	\begin{adjustbox}{width=0.9\linewidth}
		{\begin{tabular}{l c c c}
				\toprule
				Goal & R.E. (deg.) & T.E. (cm) & S.R.($\%$) \\
				\midrule
				Beachball & $2.82 \pm 0.12$ & $3.93 \pm 0.86$ & $100$ \\
				Keyboard & $5.61 \pm 3.50$ & $8.14 \pm 0.35$ & $90$ \\
				Microwave & $2.82 \pm 0.23$ & $3.59 \pm 0.42$ & $100$ \\
				Phonebook & $2.83 \pm 0.25$ & $3.37 \pm 0.95$ & $100$ \\
				\bottomrule
		\end{tabular}}
	\end{adjustbox}
\end{table}

\begin{table}[th]
	\centering
	\caption{Rotation Error (R.E.), Translation Error (T.E.), and the Success Rate (S.R.) of VIO pose estimates with MOCAP as the ground-truth reference.
        }
	\label{tab:pose_estimation_results_hardware_vio_mocap}
	\begin{adjustbox}{width=0.9\linewidth}
		{\begin{tabular}{l c c c}
				\toprule
				Goal & R.E. (deg.) & T.E. (cm) & S.R.($\%$) \\
				\midrule
				Beachball & $2.42 \pm 0.13$ & $3.87 \pm 0.89$ & $100$ \\
				Keyboard & $2.36 \pm 0.27$ & $3.28 \pm 1.20$ & $100$ \\
				Microwave & $2.36 \pm 0.14$ & $4.29 \pm 1.30$ & $100$ \\
				Phonebook & $2.40 \pm 0.23$ & $3.63 \pm 1.23$ & $100$ \\
				\bottomrule
		\end{tabular}}
	\end{adjustbox}
\end{table}



\subsubsection{Splat-Plan Evaluations}
\begin{figure*}[th]
    \centering
    \includegraphics[width=\textwidth, trim={0, 0em, 0, 0}]{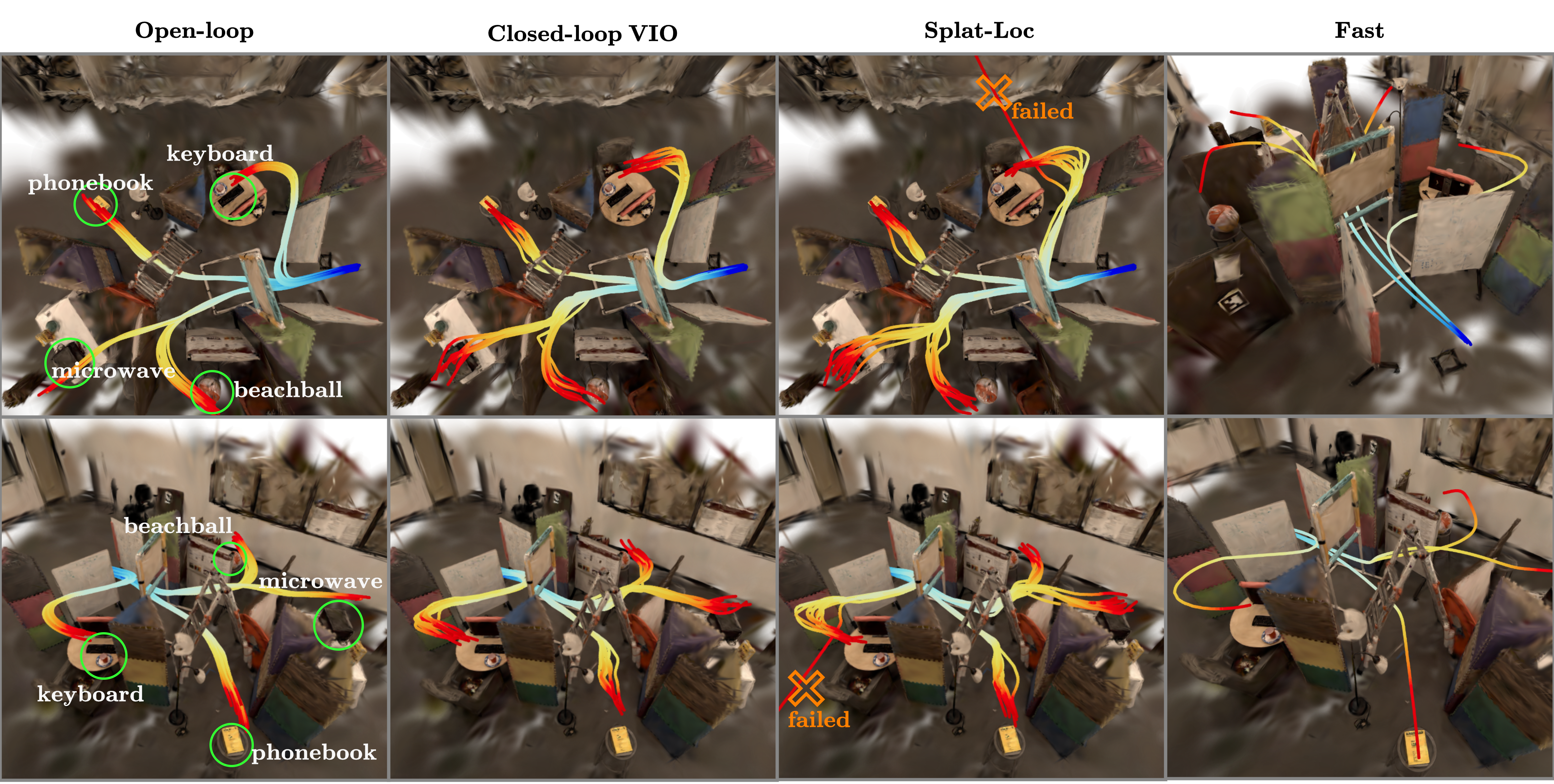}
    \caption{Ground-truth trajectories of the drone navigating, projected onto the \textbf{Maze} GSplat. The drone is subjected to different goal locations and control schemes. The recorded flight trajectories are best viewed on our website \url{https://chengine.github.io/splatnav/}.}
    \label{fig:hardware_viz}
\end{figure*}

Visualizations of 120 trajectories across four goal locations and three control schemes can be found in \cref{fig:hardware_viz}. Note that all flights were collision-free with respect to the true scene except for one flight using Splat-Loc to navigate to the keyboard. The drone was oriented toward the edge of the scene, where features were few and the GSplat quality was poor. The poor quality can be attributed to the lack of training images pointing toward the edges of the scene, as we wanted to reconstruct the foreground in the highest quality. These qualitative results indicate that, within the confines of our controlled setting, all control schemes work equally well. These results are promising for Splat-Loc from a convenience point of view. We noticed that the VIO of the drone would drift in subsequent runs, necessitating the reinitialization of the VIO at the start of every run. In addition, as the VIO is not calibrated to be in the GSplat frame, we manually aligned the frames by zero-ing the VIO of the drone at the same position for all flights and for collection of training data. Meanwhile, Splat-Loc needed no such alignment, and was kept running continuously throughout all experiments without zero-ing (even in control schemes that do not use Splat-Loc). 

\begin{figure}[th]
    \centering
    \includegraphics[width=\columnwidth, trim={0, 0em, 0, 0}, clip]{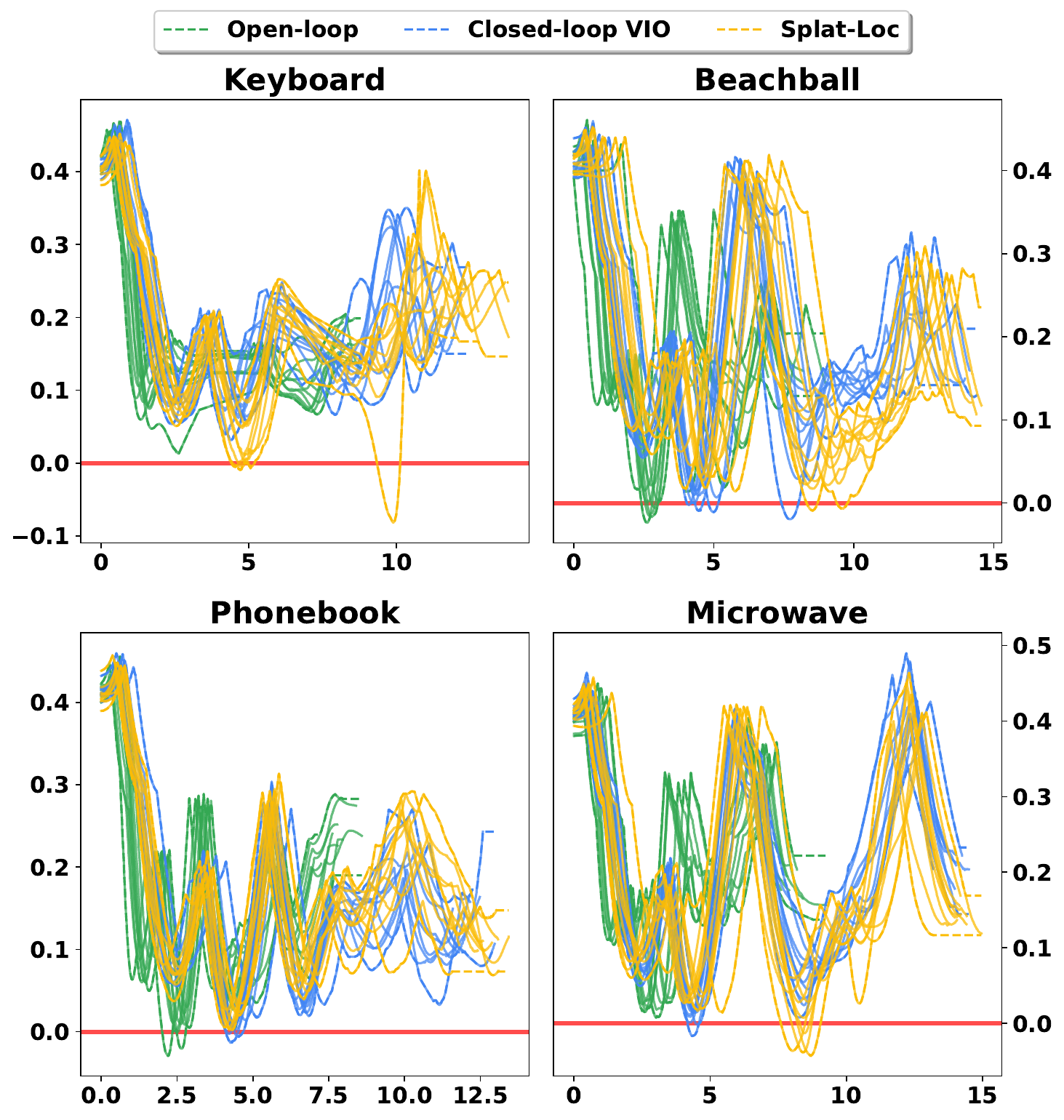}
    \caption{Distances (meters) of goal-conditioned trajectories to the Gaussian Splat as measured by the motion capture poses across three different control schemes. The abscissa represents time in seconds. The spread and the individual trajectories are visualized.}
    \label{fig:hardware_quantitative}
\end{figure}

Qualitatively, we see similar trends in \cref{fig:hardware_quantitative}. All control schemes are unsafe at different times, but in similar amounts. Note that some curves dip below 0, yet are verifiably safe in real-life. This is due to a variety of reasons, the most of prominent of which are: the difference in the set robot radius ($0.25$ cm) versus the true radius ($0.18$ cm), errors in aligning the motion capture frame into the frame of the GSplat (because the GSplat was not trained using motion capture), and the tracking capabilities of the drone. Note that the safety violation of all control schemes is relatively small compared to the size of the drone, which allows error in low-level tracking to obfuscate advantages of one method over another, especially in cluttered environments. 

\subsubsection{Fast Control}
We stress test Splat-Plan by increasing $v_{\rm max}$ until the onboard VIO could no longer track the desired waypoint with enough accuracy to avoid collision, which was $\unit[1.5]{m/s}$. These speeds, coupled with the clutter in the environment, allowed for dynamic flight, which is visualized in the right column of \cref{fig:hardware_viz}. We point readers toward the associated videos hosted on our website (\url{https://chengine.github.io/splatnav/}) to better visualize the trajectories.

\subsubsection{Closed-loop Endurance}
Finally, we stress test the Splat-Loc re-planning pipeline through endurance flights. The pipeline is left to continually execute. Once the drone reaches a goal location, another goal location is set. We demonstrate collision-free flight over the order of minutes, which can again be visualized on our website (\url{https://chengine.github.io/splatnav/}).

\section{Conclusion}
\label{sec:conclusion}
We introduce an efficient navigation pipeline termed \mbox{\emph{Splat-Nav}} for robots operating in GSplat environments. \mbox{Splat-Nav} consists of a guaranteed-safe planning module \emph{Splat-Plan}, which allows for real-time planning ($>$ \unit[2]{Hz}) by leveraging the ellipsoidal representation inherent in GSplats for efficient collision-checking and safe corridor generation, facilitating real-time online replanning. Splat-Plan demonstrates superior performance in terms of conservativeness, safety, success rate and comparable computation times compared to point-cloud and NeRF methods on the same scene. Moreover, our proposed pose estimation module \emph{Splat-Loc} computes high-accuracy pose estimates faster (\unit[25]{Hz}) and more reliably compared to existing pose estimation algorithms for radiance fields, such as NeRFs. We present extensive hardware and simulation results, highlighting the effectiveness of Splat-Nav. 

\section{Limitations and Future Work}
\label{sec:limitations_future_work}

We only tested Splat-Nav in pre-constructed scenes. 
Existing GSplat SLAM algorithms do not run in real-time \cite{matsuki2023gaussian, yugay2023gaussian}, limiting the application of our method in online mapping.
In future work, we seek to examine the derivation of real-time GSplat mapping methods, integrated with the planning and pose estimation algorithms proposed in this work.
Additionally, the results from Section~\ref{sssec:pose_estimation_mapping_localization} suggest that we can incorporate Splat-Loc as a localization module within online GSplat SLAM algorithms to improve localization accuracy and the resulting map quality.

We assumed that the pre-constructed scenes were correct.
Safety of the planned trajectories depends on the quality of the underlying GSplat map.
As noted in \cref{rem:uncertainty}, we can use different confidence levels of the ellipsoids to account for uncertainty in an object in the GSplat map.
Splat-Plan cannot do anything if an obstacle is completely missing from the scene, which is a fundamental limitation of the GSplat map representation. Likewise, Splat-Plan could fail if the initialization graph-search procedure which utilizes Dijkstra fails to find a path to the goal, which could occur in maps with a coarse resolution.
Future work will examine uncertainty quantification of different regions within a GSplat scene to aid the design of active-planning algorithms that enable a robot to collect additional observations in low-quality regions, such as areas with missing/non-existent geometry, while updating the GSplat scene representation via online mapping.

We only tested Splat-Nav in a static scene.
This could be a limitation in many practical problems.
Using NeRFs and GSplat for dynamic environments remains an open area of research, especially in scenes without prerecorded motion \cite{wu20234d, pumarola2021d, gao2021dynamic}. 
Splat-Plan is fast enough to be extended easily to problems with dynamic scenes so long as the underlying dynamic GSplat representation is available. 

The performance of Splat-Loc depends on the presence of informative features in the scene. 
We can address this in two ways: through planning and by incorporating additional sensor data.
Future work will explore the design of planning algorithms that bias the path towards feature-rich regions, improving localization accuracy during path execution.
Future work will also incorporate IMU data to improve the robustness of the pose estimator, particularly in featureless regions of the scene where the PnP-RANSAC procedure might fail.

Splat-Plan and Splat-Nav require loading the GSplat model onto the GPU, which takes up about \unit[10]{GB} of GPU memory. Many drone platforms do not have the onboard compute resources to load the GSplat model, hindering onboard computation. Future work will seek to reduce the memory-usage demands of GSplat models, e.g., using sparse GSplat models.

\appendices
\label{sec:supp_material}

\let\oldsubsection = \thesubsection
\renewcommand{\thesubsection}{\Alph{subsection}}

\ifbool{further_condense}
{
}
{
\section{Creating an Ellipsoid for a Robot}
\label{app:robot_to_ellipsoid}
One can construct an ellipsoidal representation of a robot (or obstacle in the map) given an input mesh or point cloud. A mesh of $M$ vertices can be converted to a point cloud by selecting its vertices to yield $\mathcal{P}_\mathcal{R} = \{x_i\}_{i=1}^M$. In order to convert these canonical representations to the one we desire, we can solve a simple convex program to generate our minimum-volume bounding ellipsoid $\mathcal{E}_\mathcal{R}$, given by:
\begin{equation}
\label{program:min_ellipsoid}
\begin{split}
    &\minimize{A \in \mbb{S}_{++}, b \in \mbb{R}^{3}} -\log |A|\\
    &\subj \enspace
    ||A x_i + b||_2 \leq 1,  \; \forall x_i \in \mathcal{P}_\mathcal{R} .
\end{split}
\end{equation}
Then, the mean of $\mathcal{E}_\mathcal{R}$ is given by $\mu_\mathcal{R} = -A^{-1} b$ and the covariance $\Sigma_\mathcal{R} = (A^T A)^{-1}$. Note, if the robot body consists of many points, computing the convex hull of these points first, and using the convex hull vertices in (\ref{program:min_ellipsoid}) will drastically lower the computation costs. 
}

\section{Proof of \cref{prop:ellipsoidtest2}}
\label{app:prop_gen_eigen_proof}
This proposition was stated by \cite{alger2020ellipsoidtest} without proof. We provide the proof here. The goal is to diagonalize both $\Sigma_a$ and $\Sigma_b$ in such a way that they share the same eigenbasis. In this way, the matrix inversion amounts to a reciprocal of the eigenvalues. We can find such an eigenbasis through the generalized eigenvalue problem, which solves the following system of equations ${\Sigma_a\phi_i = \lambda_i \Sigma_b \phi_i.}$
    These generalized eigenvalues and eigenvectors satisfy the identity ${\Sigma_a \phi = \Sigma_b \phi \Lambda.}$
    
    To solve the generalized eigenvalue problem, we utilize the fact that $\Sigma_a$ and $\Sigma_b$ are symmetric, positive-definite, so the Cholesky decomposition of $\Sigma_b = L L^T$,
    where $L$ is lower triangular. This also means that we can represent $\Sigma_a$ by construction as: ${\Sigma_a = L C L^T = L V \Lambda V^T L^T,}$
    where $C$ is also symmetric, positive-definite, so its eigen-decomposition is $C = V \Lambda V^T$, where the eigenvectors $V$ are orthonormal $V^T V = I$. To retrieve the original eigenbasis $\phi$, we solve the triangular system $L^T \phi = V$. A consequence of this solution is that: ${V^T V = \phi^T L L^T \phi = \phi^T \Sigma_b \phi = I.}$
    Moreover, this means that $\Sigma_b = \phi^{-T} \phi^{-1}$. We can show that $\Lambda$ and $\phi$ indeed solve the generalized eigenvalue problem through the following substitution:
    \begin{equation*}
        \Sigma_a \phi = LV\Lambda V^T L^T \phi = \underbrace{L L^T}_{\Sigma_b} \phi \Lambda \phi^T \underbrace{L L^T}_{\Sigma_b} \phi = \Sigma_b \phi \Lambda \underbrace{\phi^T \Sigma_b \phi}_{I},
    \end{equation*}
    where $\Lambda = \diag (\lambda_i)$ is the diagonal matrix of eigenvalues.
    
    As a result, we can write the matrix inversion in Theorem 1
    as the following:
    \begin{equation*}
    \begin{split}
        \left[ \frac{1}{1-s}\Sigma_a + \frac{1}{s}\Sigma_b \right] ^{-1} &= 
        \left[\frac{1}{1-s} \Sigma_b \phi \Lambda \phi^T \Sigma_b  + \frac{1}{s} \Sigma_b \right]^{-1} \\ 
        &= \left[\phi^{-T} \left(\frac{1}{1-s} \Lambda + \frac{1}{s} I \right) \phi^{-1} \right]^{-1} \\
        &= \phi \; \diag\left(\frac{s(1-s)}{1 + s(\lambda_i - 1)} \right) \; \phi^T.
    \end{split}
    \end{equation*}

\ifbool{further_condense}
{
}
{
\begin{algorithm2e} [th]
    \caption{Minimal Collision Test Set}
    \label{alg:min_testset}


    \KwIn{Query point $x$, robot's largest eigenvalue $\kappa$, ellipsoids $\mathcal{G} = \{\mu_i, \Sigma_i\}_{i=1}^N$\;}
    \KwOut{Minimal Collision Set $\mathcal{G}_x \subseteq \mathcal{G}$\;}
    $\lambda_{\rm max}(\mathcal{I}) \triangleq \max_{i \in \mathcal{I}} \lambda_{\rm max}(\Sigma_i)$\;
    KD $\leftarrow$ KDTree($\{\mu_i\}$)\;
    $\mathcal{I}_0 \leftarrow \{1, \ldots, N\}$\;
    $\mathcal{I}_1 \leftarrow$ KD.QueryBallPoint($x$, $\lambda_{\max}(\mathcal{I}_0) + \kappa$)\;
    \While{$\lambda_{\max}(\mathcal{I}_1) < \lambda_{\max}(\mathcal{I}_0)$ or $\mathcal{I}_1 \neq \varnothing$}{
        $\mathcal{I}_0 \leftarrow \mathcal{I}_1$ \;
        $\mathcal{I}_1 \leftarrow$ KD.QueryBallPoint($x$, $\lambda_{\max}(\mathcal{I}_0) + \kappa$)\;
    }
    $\mathcal{G}_x = \{\mu_i, \Sigma_i\}_{i \in \mathcal{I}_1}$\;
\end{algorithm2e}
}

\section{Proof of \cref{prop:supp_hyperplane}}
\label{app:proof_supp_hyperplane}
From \cref{prop:supp_hyperplane}, we know that we must find an $x$ such that $K(s) > 1$ for all ellipsoids in the test set $\mathcal{G}^*$.
We will approximate this safe set as a polytope.
Given a test point $x^*$ and the $j$th ellipsoid in the collision test set $\mathcal{G}^*$, we can use our collision test (Corollary 2) to derive these polytopes. Notice that $\mathcal{E}_a$ represents the position of the robot (i.e., the mean is test point $\mu_a = x^*$). For any value of ${\bar{s} \in (0, 1)}$, the safety test is a quadratic constraint of the test point $x^*$, namely $\Delta_j^T \Sigma_{x^*,j}^{-1} \Delta_j > 1$, where we omit the dependence of $\Sigma$ on $s$ for brevity.

To derive the supporting hyperplane for the safety check, we will first find the point on the ellipsoid and then linearize our test about this point.
Let $f_j(x) = (x-\mu_j)^T \Sigma_{x^*,j}^{-1} (x-\mu_j)$ for any arbitrary $x$. Note that $f_j(x^*) = \Delta^T_j \Sigma_{x^*,j}^{-1} \Delta_j = k_j^2$.
We can immediately see that point $x_0 = \mu_j + \frac{1 + \epsilon}{k_j} \Delta_j$ will be outside of the ellipsoid for any $\epsilon > 0$ since:
\begin{equation}
    \begin{aligned}
        f_j(x_0) &= \frac{(1+\epsilon)^2}{k_j^2} \Delta_j^T \Sigma_{x^*,j}^{-1} \Delta_j = (1 + \epsilon)^2 > 1.
    \end{aligned}
\end{equation}
Note that this point $x_0$ lies on the ray starting from the center of the ellipsoid $\mu_j$ and passing through the center of the robot at $x^*$.
We then linearize the constraint $f_j(x)$ about $x_0$. 
Taking the derivative yields $\frac{df_j}{dx}(x) = 2 (x-\mu_j)^T \Sigma_{x^*,j}^{-1}$.
The linear approximation of the constraint is then:
\begin{equation}
    \begin{aligned}
        f_j(x) &\approx f_j(x_0) + \left[\frac{df_j}{dx}(x_0)\right] (x - x_0) \geq 1.
    \end{aligned}
\end{equation}
Plugging in $x_0 = \mu_j + \frac{1+\epsilon}{k_j}\Delta_j$ and simplifying yields the given expression $\Delta^T_j \Sigma_{x^*,j}^{-1} x \geq {(1+\epsilon)}{k_j} + \Delta^T_j \Sigma_{x^*,j}^{-1} \mu_j$.

We need to also prove that the above constraint is a supporting hyperplane by showing that all feasible points when the constraint is equality is necessarily outside of the ellipsoid parametrized by $\Sigma_{x^*,j}^{-1}$. Recall that $x_0 = \mu_j + \frac{\Delta_j}{k_j}$ is a point both on the surface of the ellipsoid and on the hyperplane. Therefore, all feasible points on the hyperplane can be expressed as $x = x_0 + \delta$, where $\Delta_j^T \Sigma_{x^*,j}^{-1} \delta = 0$. If the hyperplane supports the ellipsoid parametrized by $\Sigma_{x^*,j}^{-1}$, then necessarily $f(x) = (x-\mu_j)^T \Sigma_{x^*,j}^{-1} (x-\mu_j) \geq 1$. For all points on the hyperplane, the ellipsoid constraint evaluates to 
\begin{align*}
    f_j(x_0 + \delta) &= \left(\frac{1+\epsilon}{k_j}\Delta_j + \delta\right)^T \Sigma_{x^*,j}^{-1} \left(\frac{1+\epsilon}{k_j}\Delta_j + \delta \right) \\
    &=\underbrace{\frac{(1+\epsilon)^2}{k_j^2}\Delta_j^T \Sigma_{x^*,j}^{-1} \Delta_j }_{ = (1+\epsilon)^2}+ \underbrace{\delta^T \Sigma_{x^*,j}^{-1} \delta}_{\geq 0} \\ 
    & \quad + \underbrace{\frac{2(1+\epsilon)}{k_j}\Delta_j^T \Sigma_{x^*,j}^{-1} \delta}_{= 0} > 1.
\end{align*}
Hence, the constraint in \cref{prop:supp_hyperplane} is a supporting hyperplane. 

\section{Proof of \cref{corr:supp-hyperplane-line}}
\label{app:proof-supp-hyperplane-line}
Our objective in this section is to show that the hyperplane created by \cref{corr:supp-hyperplane-line} always contains the line segment $\ell(t) = x_0 + t \delta_x$ for $t \in [0, 1]$. When the minimum unconstrained $t^*$ occurs outside $t \in [0, 1]$, then necessarily the constrained minimum occurs at either of the two endpoints due to convexity of $K$ with respect to $t$. Therefore, we consider the case where $t^* \in (0, 1)$ and the case where $t^* \in \{0, 1\}$. Let $x^* = x_0 + t^* \delta_x$.

\subsection{Optimum in the Interior}
Note that $t^*$ occurs when the normal of this ellipsoid $\Sigma_{x^*,b}^{-1} (x_0 + t^* \delta_x - \mu)$ is perpendicular to the direction of motion $\delta_x$. Given that $K(s^*) \geq 1$ (\cref{corr:ellipsoid-test-line}), \cref{app:proof_supp_hyperplane} states that the point $x^*$ is safe and satisfies the hyperplane constraint

\begin{equation}
    (x^* - \mu_j)^T \Sigma_{x^*,b}^{-1}(s^*) (x - \mu_j) \geq k_j^*.
\end{equation}
All points $x$ along the line can be parameterized as $x^* + \tilde{t} \delta_x$ for $\tilde{t} = t - t^*$, which due to orthogonality of $\delta_x$ to the ellipsoid normal, yields

\begin{equation}
\label{eq:proof-supp-hyperplane-line}
\begin{split} 
        &(x^* - \mu_j)^T \Sigma_{x_0,b}^{-1}(s^*) (x^* + \tilde{t} \delta_x - \mu_j) = \\
        &(x^* - \mu_j)^T \Sigma_{x_0,b}^{-1}(s^*) (x^*- \mu_j) + \tilde{t} \underbrace{(x^* - \mu_j)^T \Sigma_{x_0,b}^{-1}(s^*) \delta_x}_0 =\\
        &(x^* - \mu_j)^T \Sigma_{x_0,b}^{-1}(s^*) (x^*- \mu_j) \geq k_j^*.
\end{split}
\end{equation}
Therefore, the entire line segment is contained in the hyperplane and is not in collision. 

\subsection{Optimum on the Boundary}
Boundary optima will not necessarily yield orthogonality of $\delta_x$ with the normal of the ellipsoid $\Sigma_{x^*,j}^{-1}$. Without loss of generality, define $x^* = x_0$, which is the closest point on the line segment to the ellipsoid $\mathcal{E}_j$ since the direction of the line segment can always be flipped to achieve this result. The unconstrained optima will occur at some point $x_0 + \hat{t} \delta_x$, where $\hat{t} < 0$. Therefore, 

\begin{equation}
    \begin{split}
        &\delta_x^T \Sigma_{x^*,b}^{-1} (x_0 + \hat{t} \delta_x - \mu_j) = 0\\
        &\delta_x^T \Sigma_{x^*,b}^{-1} (x_0 - \mu_j) + \hat{t} \underbrace{\delta_x^T \Sigma_{x^*,b}^{-1}\delta_x}_{\geq 0} = 0\\
        &\therefore (x^* - \mu_j)^T \Sigma_{x^*,j}^{-1}\delta_x \geq 0.
    \end{split}
\end{equation}

Following the reasoning of \cref{eq:proof-supp-hyperplane-line}, for all points on the line $x = x^* + t \delta_x \; \forall \; t \in [0,1]$, we have $(x^* - \mu_j)^T \Sigma_{x^*,j}^{-1}(s^*)(x - \mu_j) \geq k^*_j$. Again, in this case, the line segment still remains within the polytope. Hence, we have proven \cref{corr:supp-hyperplane-line}.

\ifbool{enable_concise_mode}
{}
{
\section{Recursive Pose Estimation}


We present additional results demonstrating the performance of Splat-Loc compared to other algorithms.
Table~\ref{tab:pose_estimation_results_statues_smaller_error} shows our performance metrics in the \textbf{Statues} scene with with ${\delta_R = 20^\circ}$ and $\delta_t = \unit[0.1]{m}$. First, we see that only the Colored-ICP algorithm fails to achieve a $100\%$ success rate. Second, our algorithms \mbox{Splat-Loc-Glue} and \mbox{Splat-Loc-SIFT} yield higher-accuracy and more consistent solutions, with mean rotation and translation errors less than \unit[0.1]{deg} and \unit[7]{mm}, respectively, and lower standard deviations. In these tests, \mbox{Splat-Loc-Glue} performs the best on each metric with a computation time of about \unit[35]{ms}.

\ifbool{condensed_paper}
{
}
{
Figure~\ref{fig:pose_estimation_results_statues_smaller_error_single_run} shows a run of the algorithms where ICP and Colored-ICP achieve relatively lower errors. 
}

\begin{table*}[th]
	\centering
	\caption{Comparison of pose estimation algorithms in the \textbf{Statues} scene with $\delta_R = 20^\circ$ and $\delta_t = \unit[0.1]{m}$}
	\label{tab:pose_estimation_results_statues_smaller_error}
		{\begin{tabular}{l c c c c}
				\toprule
				Algorithm & Rotation Error (deg.) & Translation Error (m) & Computation Time (secs.) & Success Rate ($\%$) \\
				\midrule
				ICP \cite{chen1992object} & $8.42 \mathrm{e}^{1} \pm 5.58 \mathrm{e}^{1}$ & $3.83 \mathrm{e}^{0} \pm 8.51 \mathrm{e}^{0}$ & $8.60 \mathrm{e}^{-2} \pm 4.40 \mathrm{e}^{-2}$ & $100$ \\
				Colored-ICP \cite{park2017colored} & $3.02 \mathrm{e}^{1} \pm 4.58 \mathrm{e}^{1}$ & $3.46 \mathrm{e}^{-1} \pm 5.11 \mathrm{e}^{-1}$ & $6.27 \mathrm{e}^{-2} \pm 2.72 \mathrm{e}^{-2}$ & $80$ \\
				\mbox{Splat-Loc-SIFT} (ours) & $8.32 \mathrm{e}^{-2} \pm 9.37 \mathrm{e}^{-3}$ & $5.72 \mathrm{e}^{-3} \pm 7.62 \mathrm{e}^{-4}$ & $6.56 \mathrm{e}^{-2} \pm 7.16 \mathrm{e}^{-4}$ & $100$ \\
				\mbox{Splat-Loc-Glue} (ours) & $\bm{5.73 \mathrm{e}^{-2} \pm 4.82 \mathrm{e}^{-3}}$ & $\bm{3.13 \mathrm{e}^{-3} \pm 3.24 \mathrm{e}^{-4}}$ & $\bm{3.42 \mathrm{e}^{-2} \pm 2.00 \mathrm{e}^{-3}}$ & $100$ \\
				\bottomrule
		\end{tabular}}
\end{table*}

\ifbool{condensed_paper}
{
}
{
\begin{figure}[th]
	\centering
	\includegraphics[width=\linewidth, trim={0 7ex 0 12ex}, clip]{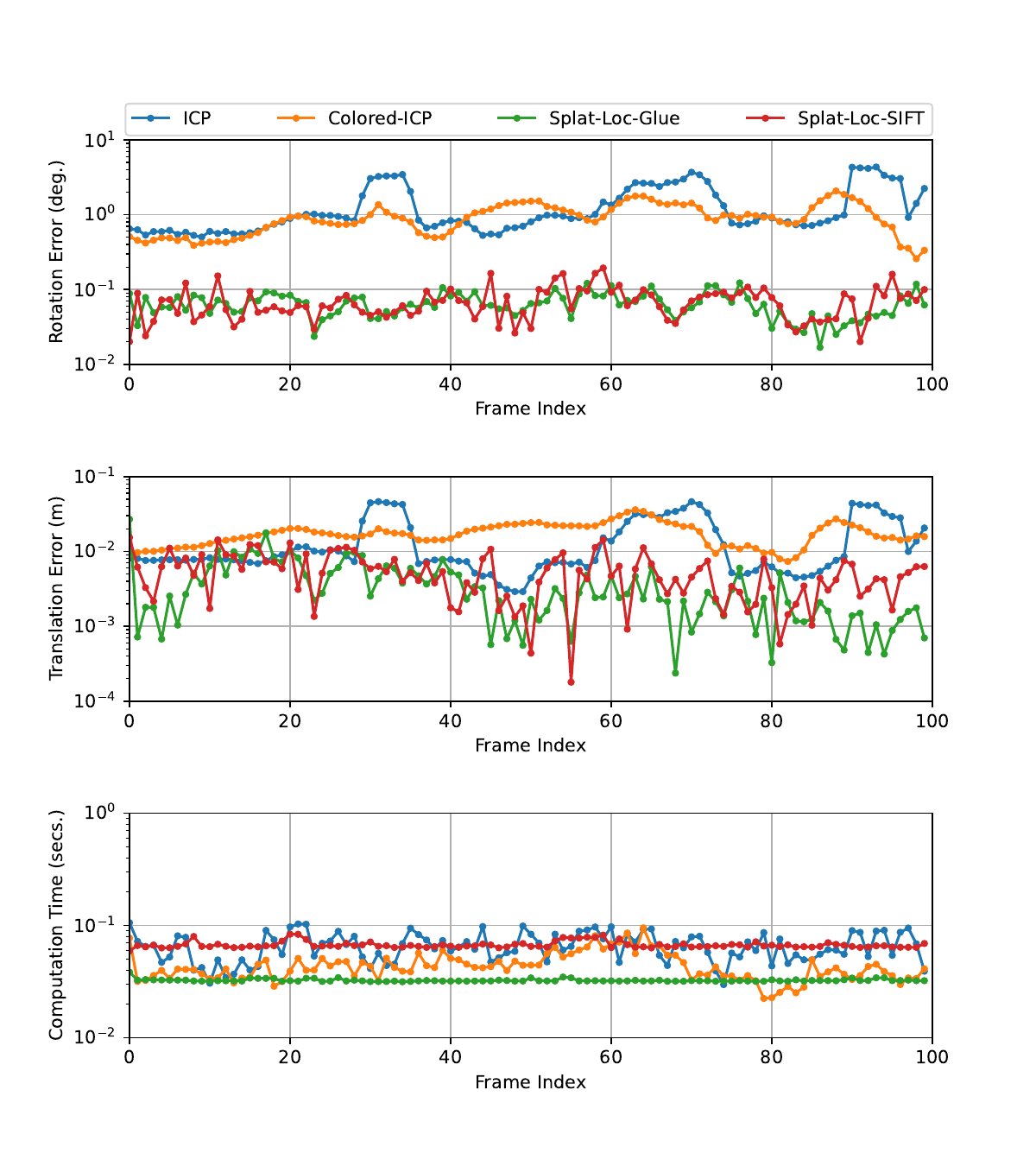}
	\caption{\textbf{Statues} Scene with $\delta_R = 20^\circ$ and $\delta_t = \unit[0.1]{m}$---Pose error and computation time per frame achieved by each pose estimation algorithm \emph{in a specific instance of the problem where ICP and Colored-ICP perform well}.}
	\label{fig:pose_estimation_results_statues_smaller_error_single_run}
\end{figure}
}

In Table~\ref{tab:pose_estimation_results_statues_larger_error}, we present results for the \textbf{Statues} scene with $\delta_R = 30^\circ$ and $\delta_t = \unit[0.5]{m}$. Note that the success rates of Colored-ICP and \mbox{Splat-Loc-SIFT} decrease sharply, suggesting that these methods require a good initial estimate of the pose of the robot. However, when they do succeed, these algorithms achieve low rotation and translation errors, with \mbox{Splat-Loc-SIFT} achieving a much-higher accuracy compared to Colored-ICP. In contrast, the \mbox{Splat-Loc-Glue} algorithm achieves a perfect success rate, along with the lowest rotation and translation errors and computation time. The results suggest the greater robustness of \mbox{Splat-Loc-Glue} compared to \mbox{Splat-Loc-SIFT}. 

\ifbool{condensed_paper}
{

}
{
Figure~\ref{fig:pose_estimation_results_statues_larger_error_single_run} shows the performance of each algorithm in an instance of the problem where all methods are successful.
}


\begin{table*}[th]
	\centering
	\caption{Comparison of pose estimation algorithms in the \textbf{Statues} scene with $\delta_R = 30^\circ$ and $\delta_t = \unit[0.5]{m}$}
	\label{tab:pose_estimation_results_statues_larger_error}
		{\begin{tabular}{l c c c c}
				\toprule
				Algorithm & Rotation Error (deg.) & Translation Error (m) & Computation Time (secs.) & Success Rate ($\%$) \\
				\midrule
				ICP \cite{chen1992object} & $9.90 \mathrm{e}^{1} \pm 3.06 \mathrm{e}^{1}$ & $1.41 \mathrm{e}^{0} \pm 2.30 \mathrm{e}^{-1}$ & $4.36 \mathrm{e}^{-2} \pm 5.08 \mathrm{e}^{-2}$ & $100$ \\
				Colored-ICP \cite{park2017colored} & $2.42 \mathrm{e}^{0} \pm 1.61 \mathrm{e}^{0}$ & $2.19 \mathrm{e}^{-1} \pm 2.02 \mathrm{e}^{-1}$ & $7.03 \mathrm{e}^{-2} \pm 2.44 \mathrm{e}^{-2}$ & $20$ \\
				\mbox{Splat-Loc-SIFT} (ours) & $7.85 \mathrm{e}^{-2} \pm 5.96 \mathrm{e}^{-3}$ & $6.28 \mathrm{e}^{-3} \pm 8.40 \mathrm{e}^{-4}$ & $7.32 \mathrm{e}^{-2} \pm 1.50 \mathrm{e}^{-3}$ & $30$ \\
				\mbox{Splat-Loc-Glue} (ours) & $\bm{6.14 \mathrm{e}^{-2} \pm 6.00 \mathrm{e}^{-3}}$ & $\bm{3.68 \mathrm{e}^{-3} \pm 6.54 \mathrm{e}^{-4}}$ & $\bm{4.28 \mathrm{e}^{-2} \pm 1.50 \mathrm{e}^{-3}}$ & $100$ \\
				\bottomrule
		\end{tabular}}
\end{table*}

\ifbool{condensed_paper}
{

}
{
\begin{figure}[th]
	\centering
	\includegraphics[width=\linewidth, trim={0 7ex 0 12ex}, clip]{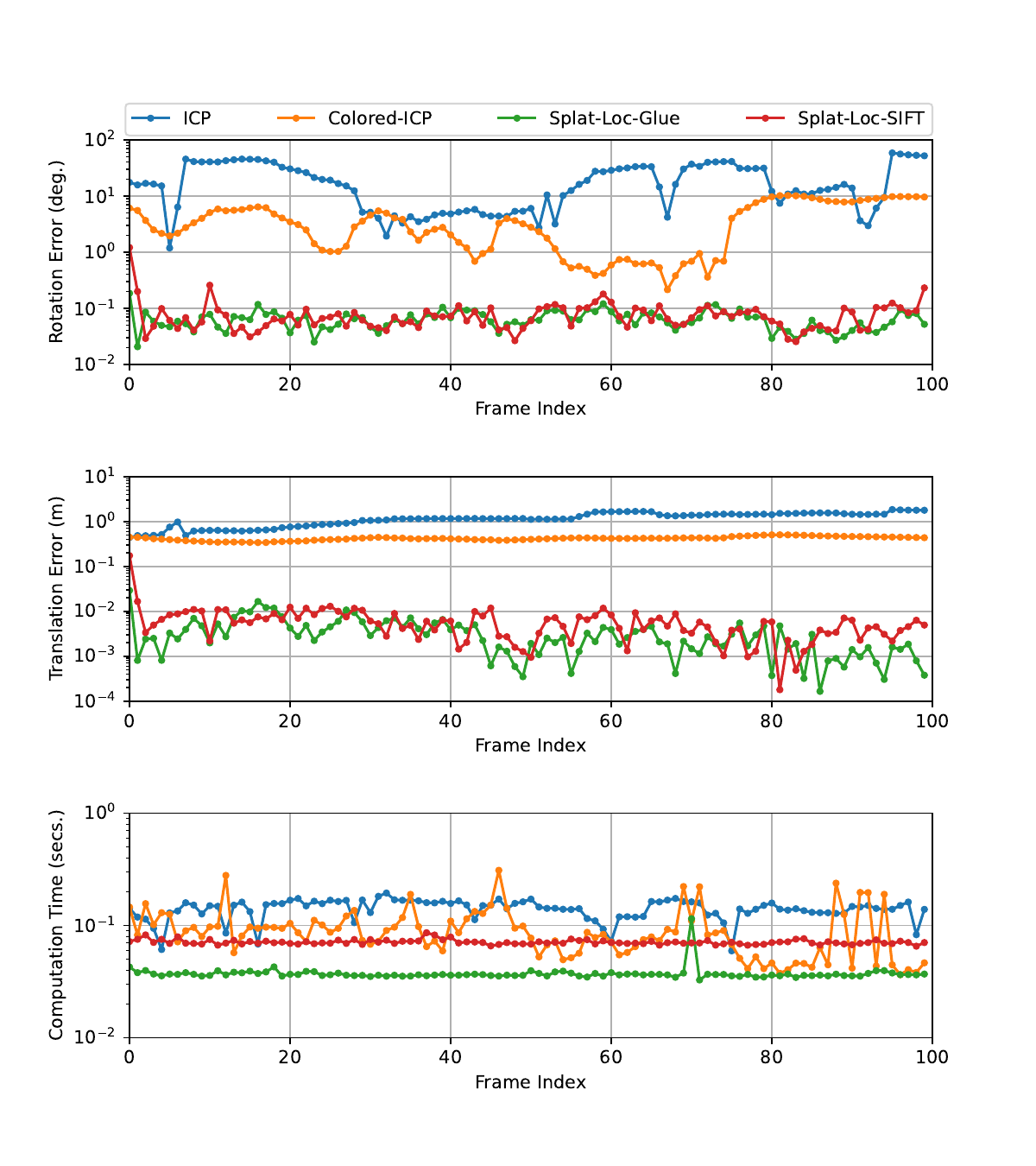}
	\caption{\textbf{Statues} Scene with $\delta_R = 30^\circ$ and $\delta_t = \unit[0.5]{m}$---Pose error and computation time per frame achieved by each pose estimation algorithm in an \emph{instance where Colored-ICP and \mbox{Splat-Loc-SIFT} are successful}.}
	\label{fig:pose_estimation_results_statues_larger_error_single_run}
\end{figure}
}

\ifbool{condensed_paper}
{
}
{
\subsection{Flightroom Scene}
Table~\ref{tab:pose_estimation_results_flightroom_larger_error} presents the results in the \textbf{Flightroom} scene with $\delta_R = 50^\circ$ and $\delta_t = \unit[0.2]{m}$. 
ICP and Colored-ICP always succeed but yield pose estimates of notably low accuracy, with mean rotation errors of about \unit[120]{deg} and \unit[125]{deg} and translation errors of about \unit[11.59]{m} and \unit[0.73]{m}, respectively. In contrast, \mbox{Splat-Loc-SIFT} fails to produce an estimate in all trials. However, \mbox{Splat-Loc-Glue} achieves a perfect success rate, yields high-accuracy pose estimates with a mean rotation error of about \unit[0.3]{deg} and a mean translation error less than \unit[5]{mm}, and takes an average of \unit[37.74]{ms} per frame.

\ifbool{condensed_paper} {

}
{
Figure~\ref{fig:pose_estimation_results_flightroom_larger_error_single_run} shows the errors in the pose estimates for a single trial. 
}


\begin{table*}[th]
	\centering
	\caption{Comparison of pose estimation algorithms in the \textbf{Flightroom} scene with $\delta_R = 50^\circ$ and $\delta_t = \unit[0.2]{m}$}
	\label{tab:pose_estimation_results_flightroom_larger_error}
	\begin{adjustbox}{width=\linewidth}
		{\begin{tabular}{l c c c c}
				\toprule
				Algorithm & Rotation Error (deg.) & Translation Error (m) & Computation Time (secs.) & Success Rate ($\%$) \\
				\midrule
				ICP \cite{chen1992object} & $1.22 \mathrm{e}^{2} \pm 2.96 \mathrm{e}^{1}$ & $1.16 \mathrm{e}^{1} \pm 2.61 \mathrm{e}^{1}$ & $1.46 \mathrm{e}^{-1} \pm 1.48 \mathrm{e}^{-1}$ & $100$ \\
				Colored-ICP \cite{park2017colored} & $1.26 \mathrm{e}^{2} \pm 3.19 \mathrm{e}^{1}$ & $7.36 \mathrm{e}^{-1} \pm 2.54 \mathrm{e}^{-1}$ & $4.85 \mathrm{e}^{-1} \pm 5.48 \mathrm{e}^{-2}$ & $100$ \\
				\mbox{Splat-Loc-SIFT} (ours) & $-$ & $-$ & $-$ & $0$ \\
				\mbox{Splat-Loc-Glue} (ours) & $\bm{2.54 \mathrm{e}^{-1} \pm 1.55 \mathrm{e}^{-1}}$ & $\bm{3.81 \mathrm{e}^{-3} \pm 2.62 \mathrm{e}^{-3}}$ & $\bm{3.78 \mathrm{e}^{-2} \pm 8.53 \mathrm{e}^{-4}}$ & $100$ \\
				\bottomrule
		\end{tabular}}
	\end{adjustbox}
\end{table*}

\ifbool{condensed_paper}
{
}
{
\begin{figure}[h!]
	\centering
	\includegraphics[width=\linewidth, trim={0 7ex 0 12ex}, clip]{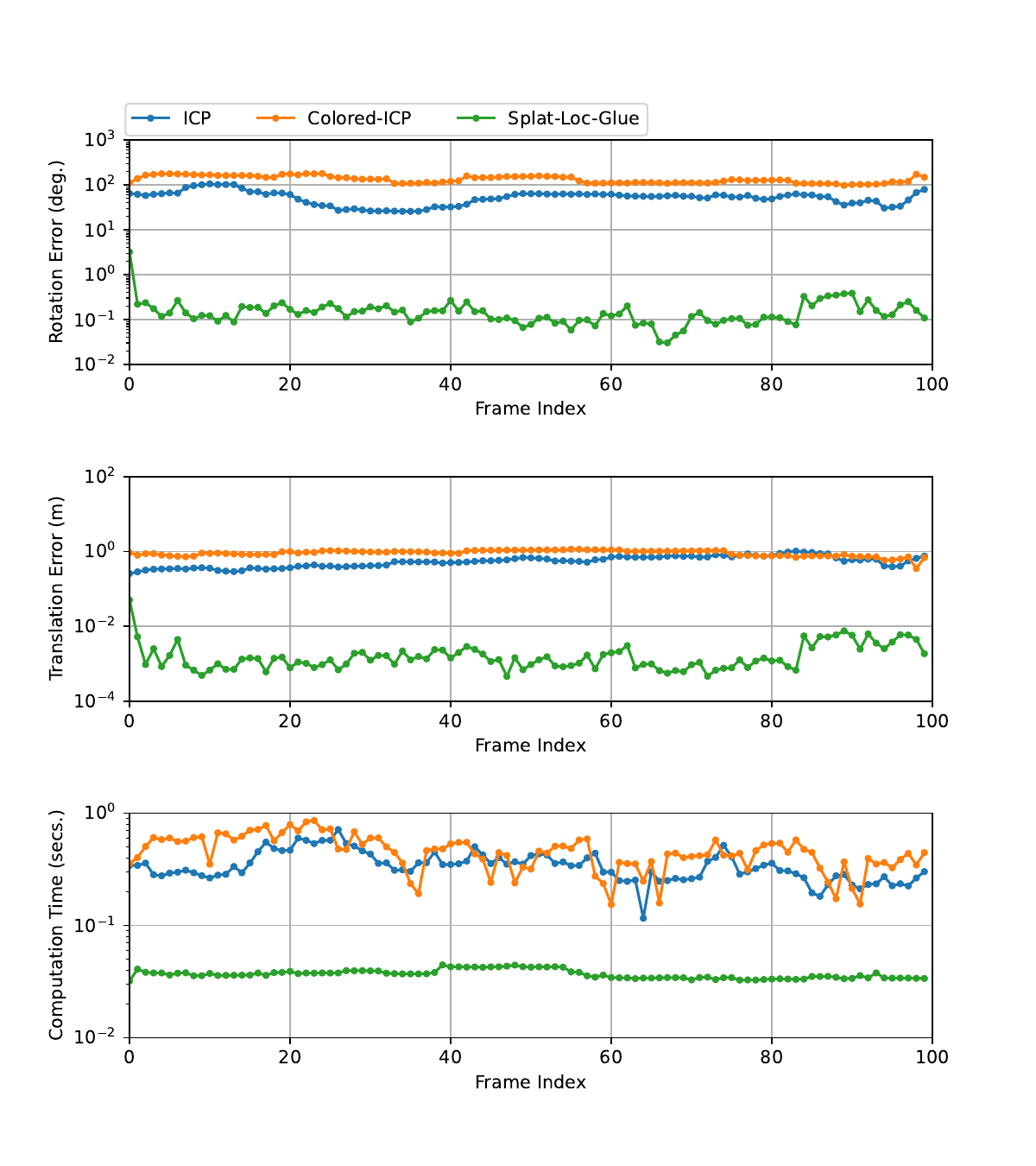}
	\caption{\textbf{Flightroom} Scene with $\delta_R = 50^\circ$ and $\delta_t = \unit[0.2]{m}$---Pose error and computation time per frame achieved by each pose estimation algorithm. \mbox{Splat-Loc-SIFT} was not successful on any trial and is thus omitted in the figure. While ICP and Colored-ICP fail to provide accurate estimates of the robot's pose, \mbox{Splat-Loc-Glue} computes high-accuracy pose estimates.}
    \label{fig:pose_estimation_results_flightroom_larger_error_single_run}
\end{figure}
}
}

\ifbool{condensed_paper}
{

}
{
Here, we examine the performance of each pose estimation algorithm in a synthetic scene \textbf{Stonehenge}, beginning with initial rotation and translation errors of $20$ deg. and $0.1$m. Table~\ref{tab:pose_estimation_results_stonehenge_smaller_error} shows the performance of each algorithm across the ten trials. Similar to the results observed with the real-world scene, \mbox{Splat-Loc-SIFT} and \mbox{Splat-Loc-Glue} yield high-accuracy pose estimates with average rotation and translation errors less than $0.5$ deg and $5$mm respectively, unlike the ICP and Colored-ICP which produce low-accuracy estimates. Moreover, ICP and Colored-ICP seem to diverge in some runs as the estimation procedure proceeds over time.
\ifbool{condensed_paper}
{}
{
, depicted in Figure~\ref{fig:pose_estimation_results_stonehenge_smaller_error_single_run}.
}%
In contrast, \mbox{Splat-Loc-Glue} achieves a perfect success rate, in contrast with \mbox{Splat-Loc-SIFT}, which achieves a $90\%$ success rate, in addition to attaining the fastest computation time.


\begin{table*}[th]
	\centering
	\caption{Comparison of pose estimation algorithms in the \textbf{Stonehenge} scene with $\delta_R = 20^\circ$ and $\delta_t = \unit[0.1]{m}$.
        }
	\label{tab:pose_estimation_results_stonehenge_smaller_error}
	\begin{adjustbox}{width=\linewidth}
		{\begin{tabular}{l c c c c}
				\toprule
				Algorithm & Rotation Error (deg.) & Translation Error (m) & Computation Time (secs.) & Success Rate ($\%$) \\
				\midrule
				ICP \cite{chen1992object} & $9.71 \mathrm{e}^{1} \pm 3.01 \mathrm{e}^{1}$ & $1.09 \mathrm{e}^{0} \pm 7.96 \mathrm{e}^{-1}$ & $3.61 \mathrm{e}^{-1} \pm 5.04 \mathrm{e}^{-2}$ & $100$ \\
				Colored-ICP \cite{park2017colored} & $7.19 \mathrm{e}^{1} \pm 3.23 \mathrm{e}^{1}$ & $5.18 \mathrm{e}^{-1} \pm 2.17 \mathrm{e}^{-1}$ & $3.68 \mathrm{e}^{-1} \pm 6.38 \mathrm{e}^{-2}$ & $100$ \\
				\mbox{Splat-Loc-SIFT} (ours) & $1.97 \mathrm{e}^{-1} \pm 2.69 \mathrm{e}^{-2}$ & $3.04 \mathrm{e}^{-3} \pm 4.27 \mathrm{e}^{-4}$ & $1.40 \mathrm{e}^{-1} \pm 5.73 \mathrm{e}^{-3}$ & $90$ \\
				\mbox{Splat-Loc-Glue} (ours) & $\bm{2.12 \mathrm{e}^{-1} \pm 2.27 \mathrm{e}^{-1}}$ & $\bm{2.97 \mathrm{e}^{-3} \pm 2.42 \mathrm{e}^{-3}}$ & $\bm{4.56 \mathrm{e}^{-2} \pm 1.12 \mathrm{e}^{-3}}$ & $100$ \\
				\bottomrule
		\end{tabular}}
	\end{adjustbox}
\end{table*}

\ifbool{condensed_paper}
{
}
{
\begin{figure}[th]
	\centering
	\includegraphics[width=\linewidth, trim={0 7ex 0 12ex}, clip]{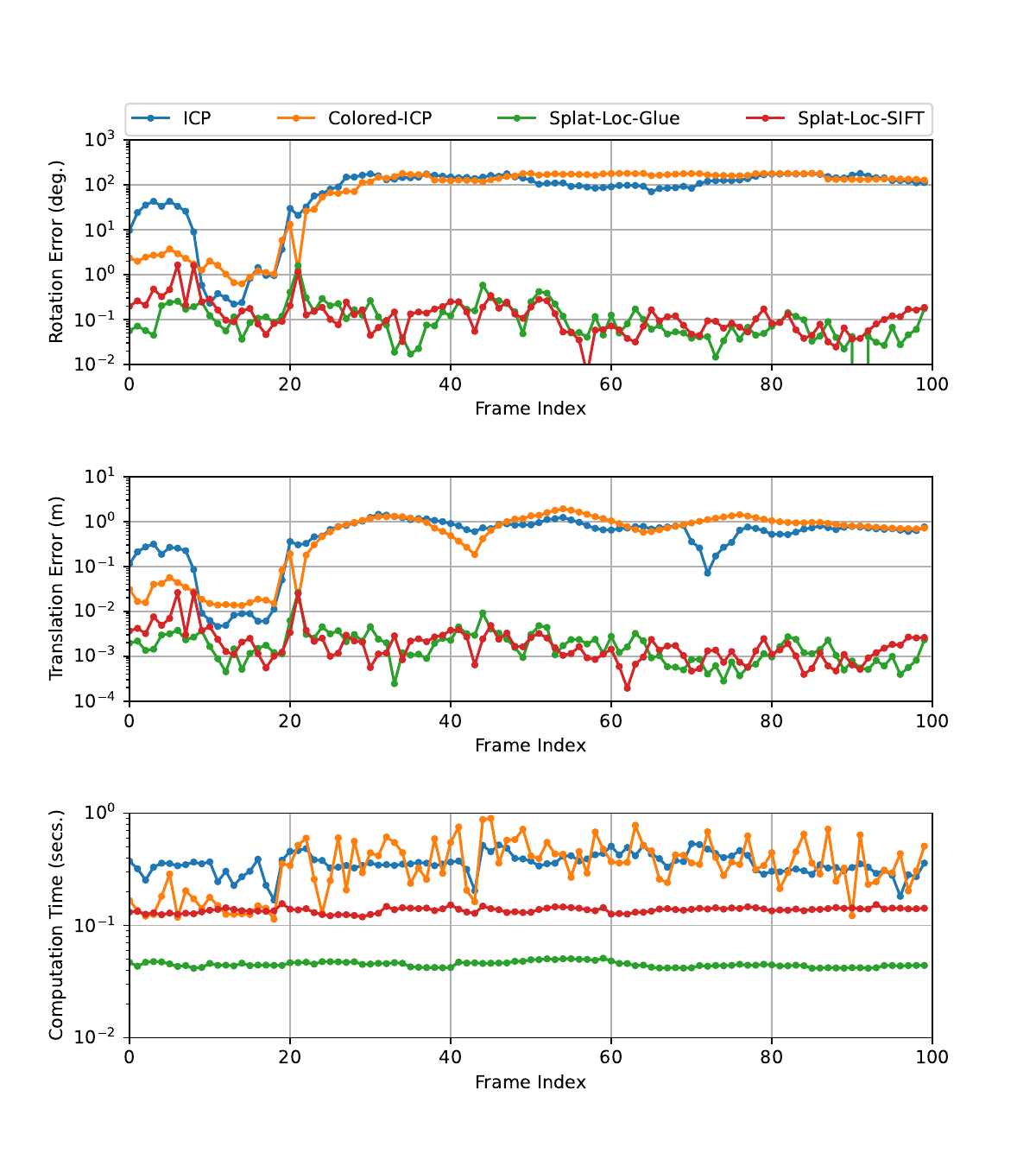}
	\caption{\textbf{Stonehenge} Scene with $\delta_R = 20^\circ$ and $\delta_t = \unit[0.1]{m}$---While ICP and Colored-ICP diverge over time, \mbox{Splat-Loc-SIFT} and \mbox{Splat-Loc-Glue} do not diverge, producing high-accuracy estimates.}
	\label{fig:pose_estimation_results_stonehenge_smaller_error_single_run}
\end{figure}
}
}

\ifbool{condensed_paper}
{
In addition, we examine the performance of the pose estimation algorithms in problems with a larger error in the initial estimate of the pose, with $\delta_R = 30^\circ$ and $\delta_t = \unit[0.5]{m}$ in the synthetic \textbf{Stonehenge} scene. We present the performance of each algorithm on each metric in Table~\ref{tab:pose_estimation_results_stonehenge_larger_error}, where we note that ICP and Colored-ICP do not provide accurate estimate of the robot's pose. Moreover, the pose estimation errors achieved by ICP and Colored-ICP have a significant variance. In contrast, \mbox{Splat-Loc-SIFT} and \mbox{Splat-Loc-Glue} yield pose estimates of high accuracy with average rotation and translation errors less than $0.5$ deg. and $5$mm, respectively, similar to the setting with low errors in the initial pose estimate. However, \mbox{Splat-Loc-SIFT} achieves a lower success rate, compared to \mbox{Splat-Loc-Glue}, which achieves a perfect success rate. 

\ifbool{condensed_paper}
{
}
{
In Figure~\ref{fig:pose_estimation_results_stonehenge_larger_error_single_run}, we show the estimation errors for a single run, highlighting the high variance in the errors associated with the pose estimates computed Colored-ICP. ICP fails to yield a pose estimate of relatively high accuracy for the duration of the problem.
}

\begin{table*}[th]
	\centering
	\caption{Comparison of pose estimation algorithms in the \textbf{Stonehenge} scene with $\delta_R = 30^\circ$ and $\delta_t = \unit[0.5]{m}$.
        }
	\label{tab:pose_estimation_results_stonehenge_larger_error}
		{\begin{tabular}{l c c c c}
				\toprule
				Algorithm & Rotation Error (deg.) & Translation Error (m) & Computation Time (secs.) & Success Rate ($\%$) \\
				\midrule
				ICP \cite{chen1992object} & $1.31 \mathrm{e}^{2} \pm 2.26 \mathrm{e}^{1}$ & $3.70 \mathrm{e}^{0} \pm 5.54 \mathrm{e}^{0}$ & $1.22 \mathrm{e}^{-1} \pm 1.53 \mathrm{e}^{-1}$ & $100$ \\
				Colored-ICP \cite{park2017colored} & $9.49 \mathrm{e}^{1} \pm 5.13 \mathrm{e}^{1}$ & $5.74 \mathrm{e}^{-1} \pm 2.88 \mathrm{e}^{-1}$ & $4.88 \mathrm{e}^{-1} \pm 1.04 \mathrm{e}^{-1}$ & $20$ \\
				\mbox{Splat-Loc-SIFT} (ours) & $\bm{2.17 \mathrm{e}^{-1} \pm 3.69 \mathrm{e}^{-2}}$ & $3.34 \mathrm{e}^{-3} \pm 5.63 \mathrm{e}^{-4}$ & $1.39 \mathrm{e}^{-1} \pm 3.15 \mathrm{e}^{-3}$ & $70$ \\
				\mbox{Splat-Loc-Glue} (ours) & $2.20 \mathrm{e}^{-1} \pm 2.03 \mathrm{e}^{-1}$ & $\bm{3.15 \mathrm{e}^{-3} \pm 2.10 \mathrm{e}^{-3}}$ & $\bm{4.51 \mathrm{e}^{-2} \pm 6.11 \mathrm{e}^{-4}}$ & $100$ \\
				\bottomrule
		\end{tabular}}
\end{table*}


\ifbool{condensed_paper}
{
}
{
\begin{figure}[t!]
	\centering
	\includegraphics[width=\linewidth, trim={0 7ex 0 12ex}, clip]{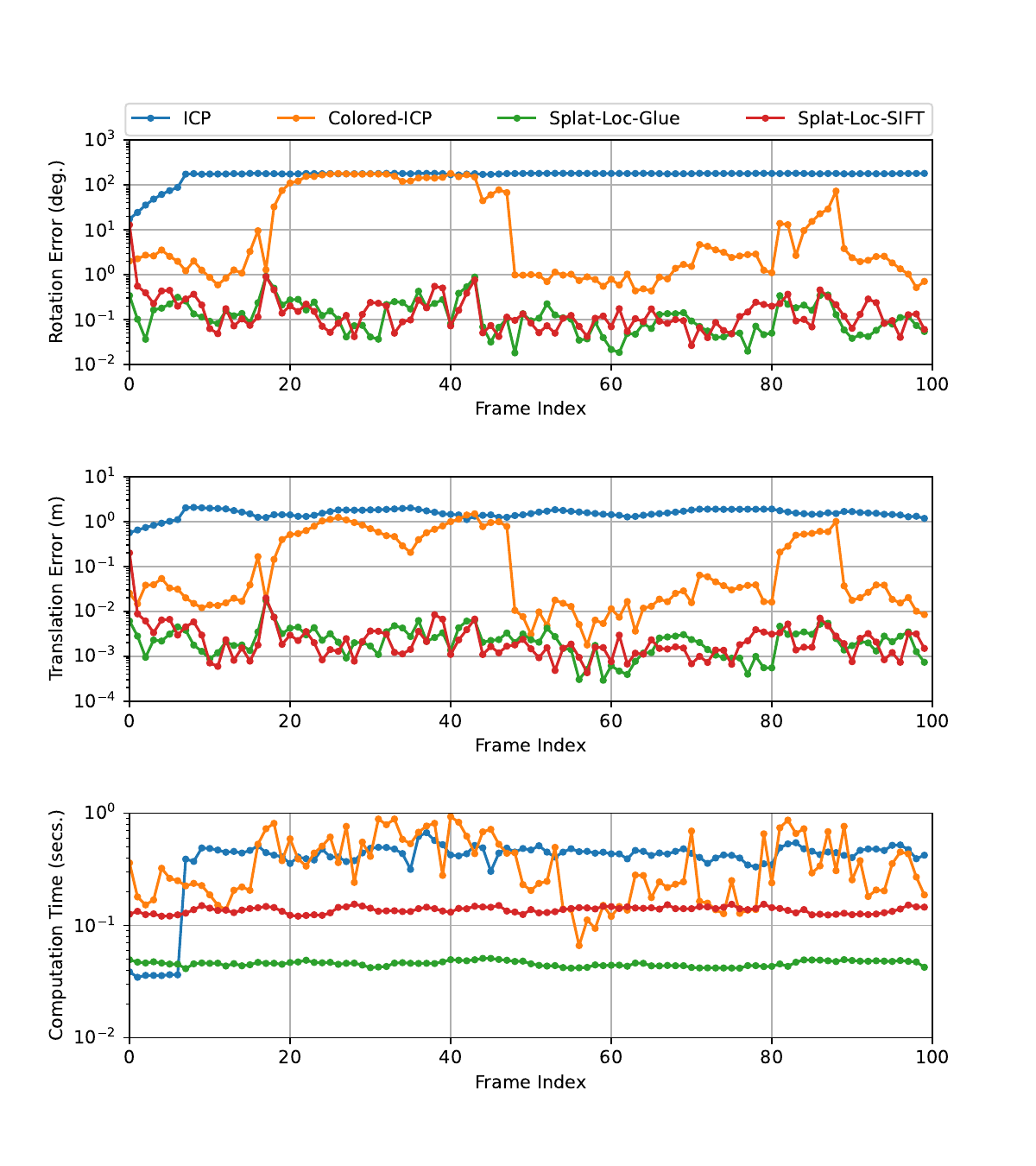}
	\caption{\textbf{Stonehenge} Scene with $\delta_R = 30^\circ$ and $\delta_t = \unit[0.5]{m}$---\mbox{Splat-Loc-SIFT} and \mbox{Splat-Loc-Glue} produce pose estimates of the highest accuracy, compared to ICP and Colored-ICP. Further, \mbox{Splat-Loc-Glue} requires the least computation time per frame.}
	\label{fig:pose_estimation_results_stonehenge_larger_error_single_run}
\end{figure}
}
}
{}
}

\ifbool{enable_concise_mode}
{
\section{Global Pose Initialization of Splat-Loc}
We evaluate the global point-cloud alignment initialization procedure from Section~\ref{sec:global_initialization} when utilized by \mbox{Splat-Loc-Glue}, in the real-world scene Statues and synthetic scene Stonehenge, with the results in Table~\ref{tab:pose_estimation_results_global_registration}. We see that the success rate is 80\% in both scenes, and it runs at approximately \unit[20-30]{Hz}. When it does succeed, the pose errors are on the order of \unit[0.3]{deg} and \unit[1]{cm}. Although these estimates are relatively good, we note that these are less accurate than the recursive pose estimates discussed above.

\begin{table}[h!]
	\centering
	\caption{Performance of the pose initialization module in Splat-Loc.}
	\label{tab:pose_estimation_results_global_registration}
	\begin{adjustbox}{width=\linewidth}
		{\begin{tabular}{l c c c c}
				\toprule
				Scene & R.E. (deg.) & T.E. (cm) & C.T. (msec.) & S.R. ($\%$) \\
				\midrule
				Statues  & $0.37 \pm 0.72$ & $1.34 \pm 2.33$ & $33.9 \pm 0.57$ & $80$ \\
				Stonehenge & $0.21 \pm 0.21$ & $0.30 \pm 0.23$ & $45.8 \pm 1.70$ & $80$ \\
				\bottomrule
		\end{tabular}}
	\end{adjustbox}
\end{table}

}
{}

\ifbool{condensed_paper}
{
}
{
\begin{figure}[th]
    \centering
    \includegraphics[width=0.49\textwidth]{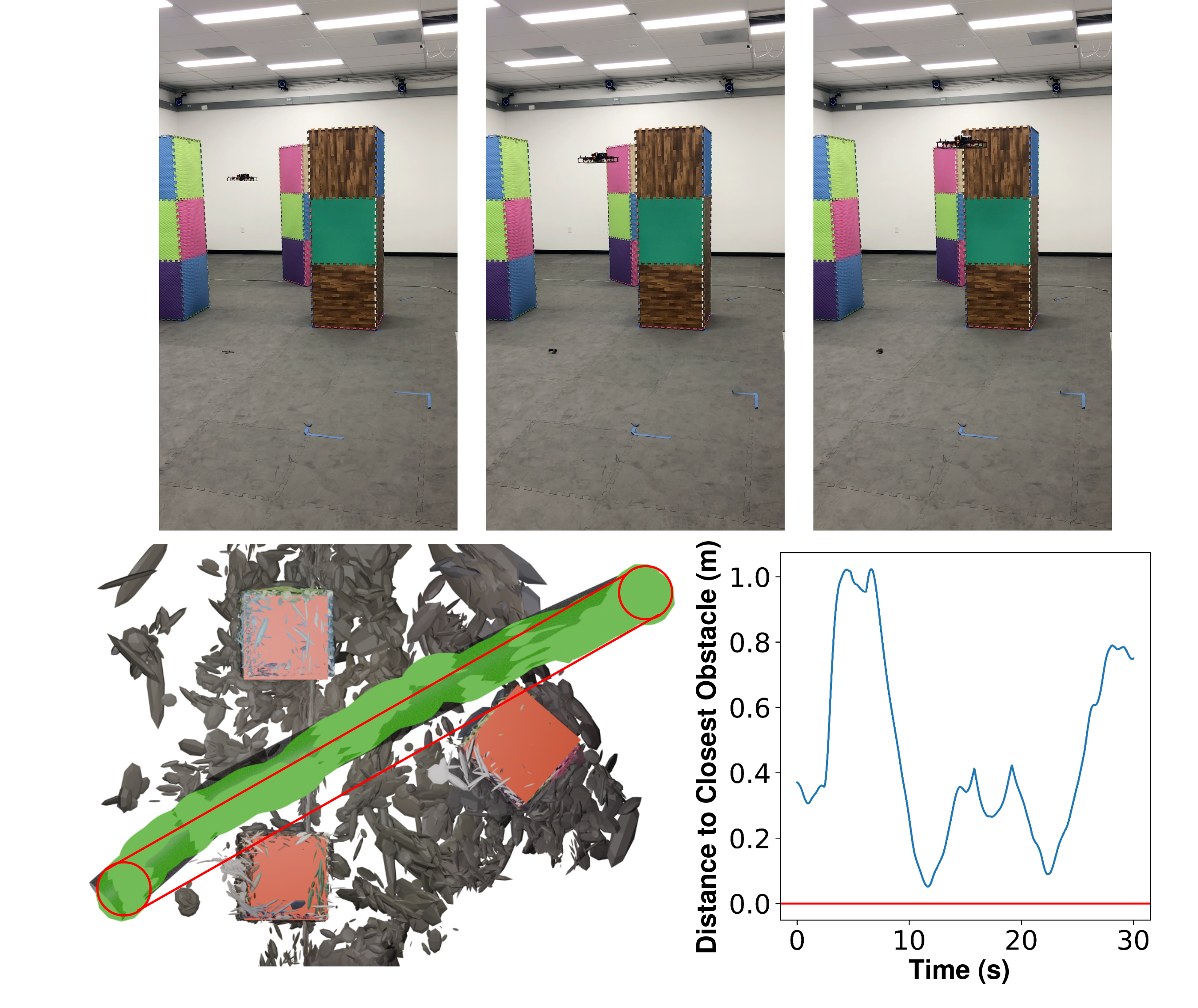}
    \caption{(top) Deployment of our safe planner, Splat-Plan, on a drone. (bottom-left) Visualization of the executed trajectory (green) and the commanded trajectory (red outline). A straight line path from start to goal is not safe. (bottom-right) Distance of the drone to the closest obstacle, taking into account the extent of the drone. Note that the executed trajectory lies above the red line and is therefore safe.}
    \label{fig:hardware}
\end{figure}
}

\ifbool{condensed_paper}
{
}
{
\section{Robustness to Slight Changes in the Scene}
\label{app:robustness_to_changes}
Given the RANSAC procedure used in computing matching features, our pose estimator provides some level of robustness to outliers in the feature matching process as well as robustness to changes in the scene, provided that these changes do not overly influence the computation of the feature correspondences. In our hardware experiments, the Mini-office scene changed slightly (unintentionally) from the time a video of the scene was recorded for training the GSplat map and the time of the experiment, which we show in \cref{fig:scene_changes}. These changes in the map did not degrade the performance of our pose estimator.

\begin{figure}[th]
    \centering
    \includegraphics[width=0.49\columnwidth]{figures/results/scene_changes/pre_experiment_scene_changes.pdf}
    \includegraphics[width=0.49\columnwidth]{figures/results/scene_changes/experiment_scene_changes.pdf}
    \caption{(left) The scene at the time the video was recorded for training the Gaussian Splats. (right) The scene at the time of the experiments was slightly different from the training scene due to unintentional changes.}
    \label{fig:scene_changes}
\end{figure}
}

\let\thesubsection = \oldsubsection

\ifbool{enable_concise_mode}
{}
{
\section{Pose Estimation}
\label{sec:estimation}

{
\color{black}
In this section, we present our pose estimation module, Splat-Loc, for localizing a robot in a Gaussian Splatting representation of its environment. First, we consider the case when an initial guess of the robot's pose is available, before presenting approaches for dealing with cases where no prior information on its pose is available. 

\subsection{Problem Formulation}
We perform pose estimation on the \SE${(3)}$ manifold, which represents rigid transformations in $3$D space. A pose in \SE${(3)}$ is parameterized by a rotation matrix ${R \in \SO(3)}$ and a translation vector ${\rho \in \mbb{R}^{3}}$.
Our approach leverages the lie algebra $\mathfrak{se}(3)$ associated with the \SE${(3)}$ matrix lie group.
$\mathfrak{se}(3)$ consists of a vectorspace of matrices in ${\mbb{R}^{4 \times 4}}$ spanned by a basis of six generators \cite{eade2013lie}. We parameterize each element of $\mathfrak{se}(3)$ by ${\xi \in \mbb{R}^{6}}$, where
${\xi = [r^{T}, p^{T}]^T}$, ${p, r \in \mbb{R}^{3}}$, with $p$ representing the translational component and $r$ the rotational component.
\ifbool{condensed_paper}
{
}
{
To simplify the subsequent discussion, we consider ${\xi \in \mathfrak{se}(3)}$, implying its application to the generators. 
We can relate the elements $\xi$ of $\mathfrak{se}(3)$ to the elements $C$ of \SE${(3)}$ through the exponential map:
\begin{equation}
    \label{eq:exponential_map}
    C = \exp(\xi) = \exp\left(
        \begin{array}{c | c}
            r_{\times} & p \\
            \hline
            \mb{0}_{1\times3} & 0
        \end{array}
    \right),
\end{equation}
where $\exp$ denotes the matrix exponential and ${r_{\times} \in \mbb{R}^{3 \times 3}}$ denotes the skew-symmetric matrix defined by $r$.
}
We assume that the state of the robot is given by ${x \in \mbb{R}^{12}}$, comprising of the lie algebra $\xi$ and the linear and angular velocities, ${\nu \in \mbb{R}^{3}}$ and ${\omega \in \mbb{R}^{3}}$, respectively, with ${x = [\xi^{T}, \nu^{T}, \omega^{T}]^{T}}$.
We assume that the discrete-time dynamic model ${f: \mbb{R}^{n} \mapsto \mbb{R}^{n}}$ of the robot is of the form: ${x_{t} = f(x_{t - 1}, u_{t - 1}, w_{t - 1}),}$
where ${w_{t} \in \mbb{R}^{n}}$ denotes Gaussian white noise at time $t$ with covariance ${Q_{t} \in \mbb{R}^{n \times n}}$ and ${u_{t} \in \mbb{R}^{m}}$ denotes the control inputs of the robot. 

The measurement model ${h: \mbb{R}^{n} \times \mbb{R}^{k} \mapsto \mbb{R}^{m}}$ of the robot can be expressed in the form: ${y_{t} = h(x_{t}, v_{t}),}$
where ${y_{t} \in \mbb{R}^{m}}$ denotes the measurement and
${v_{t} \in \mbb{R}^{n}}$ denotes Gaussian white noise at time $t$ with covariance ${R_{t} \in \mbb{R}^{n \times n}}$. 
In general, the measurements $y_{t}$ obtained by the robot can include the robot's pose, linear velocities, and its body rates.
In this work, we consider an RGB (or RGB-D) camera-based measurement model that uses keypoints from an image ${I_{t} \in \mbb{R}^{H \times W \times 3}}$ to measure elements from the robot's state.
Let ${\psi: \mbb{R}^{H \times W \times 3} \rightarrow \mbb{R}^{k}}$ denote the function that goes from a raw camera image to a set of keypoints with descriptors $\{k_{t,i}\}_i = \psi(I_{t})$.
Based on the state of the robot $x_t$, we can also render an image $\hat{I}_{t}(x_t)$ from the GSplat map and use $\psi$ to obtain another set of keypoints $\{\hat{k}_{t,j}\}_j = \psi(\hat{I}_{t}(x_t))$.
Note, each keypoint $\hat{k}_{t}$ has an associated 3D position from the GSplat map.
The measurement function $h$ can then use matches between the sets of keypoints $\{k_{t,i}\}_i$ and $\{\hat{k}_{t,j}\}_j$ to, for example, estimate the pose by solving the Perspective-$n$-Point (PnP) problem or estimate the velocities using optical flow (see \cref{sec:observation_model} for more detail).
We note that our discussion extends directly to other measurement models.


\subsection{MAP Pose Estimate}
We formulate a constrained \emph{maximum a-posteriori} (MAP) optimization problem to compute an estimate of the robot's pose that satisfies the safety constraints (defined by the polytope) at each timestep. 
We assume we have a prior estimate of the pose at time $t-1$ represented by a Gaussian distribution ${\mcal{N}(\mu_{t - 1 | t - 1}, P_{t - 1 | t - 1})}$ and know the current safe polytope containing the robot, which is parameterized by ${(A_{t}, b_{t})}$, and the polytope at the previous timestep ${(A_{t -1}, b_{t - 1})}$.
Given the dynamics and measurement models,
we can compute the new estimate of the robot's pose that maximizes the posterior distribution of the robot's pose (given the set of measurements) from:
\begin{subequations} \label{eq:map_problem}
    \begin{align}
        \min_{x_{t}, x_{t - 1}} &\norm{x_{t} - f(x_{t - 1}, u_{t - 1})}_{Q_{t - 1}^{-1}}^{2} + \norm{\psi(I_{t}) - h(x_{t})}_{R_{t}^{-1}}^{2} \nonumber \\
        & + \norm{x_{t - 1} - \mu_{t - 1 | t - 1}}_{P_{t - 1 | t - 1}^{-1}}^{2}, \label{eq:map_obj} \\
        \subj \nonumber \\
              \text{Prior} & \text{ Knowledge at time $t$:}& \nonumber \\
              &\norm{A_{t, (i)}^{T}}_{\bar{P}_{t | t - 1}}^{2} + A_{t, (i)}x_{t} \leq b_{t, (i)}, \ \forall i, \\
              \text{Prior} & \text{ Knowledge at time $t - 1$:}& \nonumber\\
              &\norm{A_{t - 1, (i)}^{T}}_{\bar{P}_{t - 1 | t - 1}}^{2} + A_{t - 1, (i)}x_{t - 1} \leq b_{t - 1, (i)}, \ \forall i,
    \end{align}
\end{subequations}
where $A_{t, (i)}$ denotes the $i$th row of $A_{t}$. In the last-two constraints in \eqref{eq:map_problem}, we encode that the $\gamma\%$-confidence ellipsoid of the estimated state at time $t$ lies within the associated safe polytope, where: ${\bar{P}_{t | t - 1} = \chi_{3}^{2}(\gamma) \left(F_{t} P_{t | t - 1} F_{t}^{T} + Q_{t - 1} \right)}$
represents the predicted covariance at time $t$, analogous to the Extended Kalman Filter (EKF) Predict Procedure, with $F_{t}$ representing the Jacobian of the dynamics model at time $t$,  and ${\bar{P}_{t - 1 | t - 1} = \chi_{3}^{2}(\gamma) P_{t - 1 | t - 1}}$. We utilize $\bar{P}_{t | t - 1}$ in place of $\bar{P}_{t | t}$ since $\bar{P}_{t | t}$ is not known a-priori.
The MAP problem in \eqref{eq:map_problem} yields the solution: ${\mb{x}_{t}^{\star} = [\mu_{t | t}^{T}, \mu_{t - 1 | t}^{T}]^{T},}$
where the estimated pose of the robot at time $t$ is given by the posterior distribution ${\mcal{N}(\mu_{t | t }, P_{t | t})}$, with $P_{t | t}$ computed from the inverse Hessian of \eqref{eq:map_obj}.

\begin{remark}
    We note that the minimizer of \eqref{eq:map_problem} represents an estimate of the mode of the posterior distribution. Here, we approximate the posterior distribution over $(x_{t}, x_{t - 1})$ by a Normal distribution, whose mean corresponds to the solution of \eqref{eq:map_problem} and whose covariance is given by the inverse of the Hessian of \eqref{eq:map_obj} evaluated at the optimal solution. Notably, this approximation is exact when the dynamics and measurement models are linear-Gaussian in the unconstrained case. Moreover, given the posterior distribution, we can transform the estimated pose of the robot from $\mathfrak{se}(3)$ to \SE${(3)}$ via the exponential map. 
\end{remark}

\begin{remark}[Smoothing of Prior Pose Estimates]
    From \eqref{eq:map_problem}, we obtain a smoothed estimate of the robot's pose at time ${t - 1}$, given by ${\mcal{N}(\mu_{t - 1 | t }, P_{t - 1 | t})}$. By incorporating measurements up to time $t$, the smoothed estimate improves upon the estimate of the robot's pose at the preceding timestep.
\end{remark}


Given a set of correspondences between keypoints in successive images, we can interpret the second term in the objective function in \eqref{eq:map_problem} as the reprojection error, resulting in a nonlinear least-squares problem, which could be challenging to solve in real-time. In the subsequent discussion, we present a faster pose estimator to reduce the computation time for solving \eqref{eq:map_problem}, at the expense of exactness.


\subsection{Handling Highly Nonlinear Dynamics}
For robots with arbitrary nonlinear dynamics models, solving the MAP optimization problem in \eqref{eq:map_problem} often proves challenging. To circumvent this challenge, we present a two-phase pose estimation procedure. In the first phase, we utilize an Unscented Kalman Filer (UKF) \cite{julier1997new} to compute an estimate of the robot's pose consisting of a \emph{predict} procedure, dependent on the robot's dynamics model, followed by an \emph{update} procedure, which performs a correction on the estimated pose given the robot's observations (i.e., camera images). In the second phase, we solve a constrained optimization problem to satisfy an approximation of the problem in \eqref{eq:map_problem}. We sumarize the procedures in Algorithm~\ref{alg:constrained_UKF}.

\begin{algorithm2e}[th]
    \caption{Splat-Loc}
    \label{alg:constrained_UKF}


    \KwIn{Prior ${\mcal{D}_{t - 1} = (\mu_{t - 1 | t - 1}, P_{t - 1 | t - 1})}$, Image $I_{t}$\;}
    \KwOut{Posterior $(\mu_{t | t}, P_{t | t })$\;}
    
    \tcp{Prediction Phase}
    $(\mu_{t | t - 1}, P_{t | t - 1}) \gets$ UKF\_Prediction$(\mcal{D}_{t - 1})$ \; 

    \tcp{Compute $y_{t}$ from $I_{t}$.}
    ${y_{t} \gets}$ Procedure \eqref{eq:pnp_problem}\;
    
    \tcp{Update Phase}
    $(\hat{\mu}_{t | t }, \hat{P}_{t | t - 1}) \gets$ Procedure \eqref{eq:gaussian_estimation_equations}\;
    
    \tcp{Constrained Estimation}
    $(\mu_{t | t}, P_{t | t}) \gets$ Procedure \eqref{eq:projection_based_pose_estimator}\;
\end{algorithm2e}

Given a prior estimate of the robot's pose ${(\mu_{t - 1| t - 1}, P_{t - 1| t - 1})}$, we execute the Prediction Procedure of the UKF, which involves computing the sigmapoints associated with the prior, propagating these sigmapoints forward using the robot's dynamic models, and subsequently, computing the predicted distribution of the robot's pose ${(\mu_{t | t - 1}, P_{t | t - 1})}$ from the propagated sigmapoints. When the robot obtains new measurements, we compute the expected measurement $\mu_{y, t}$, along with the covariance matrices $P_{yy, t}$ and $P_{xy, t}$ from the sigmapoints associated with the prediction distribution.
We fuse these new measurements via the UKF Update Procedure to compute the posterior distribution of the robot's pose from:
\begin{equation}
    \label{eq:gaussian_estimation_equations}
    \begin{aligned}
        \hat{\mu}_{t | t} &= \mu_{t | t - 1} + K_{t} \Phi(y_{t}, \mu_{y, t}),
    \end{aligned}
\end{equation}
where $K_{t}$ denotes the Kalman Gain, with mean $\hat{\mu}_{t | t}$ and covariance $\hat{P}_{t | t}$, approximating the posterior distribution, where ${\Phi}$ computes the \emph{innovation}, i.e., the difference between the observed measurement and the expected measurement.

We note that the resulting estimated mean of the robot's pose may not reside within a safe polytope. As a result, we project the distribution into the associated safe polytope by solving the constrained convex optimization problem:
\begin{subequations}
    \label{eq:projection_based_pose_estimator}
    \begin{align}
        \minimize{\mu_{t | t} \in \mbb{R}^{n}, P_{t | t} \in \mbb{S}_{++}} & -\log(\lvert P_{t | t} \rvert) + \norm{\mu_{t | t} - \hat{\mu}_{t | t}}_{\hat{P}_{t | t}^{-1}}^{2} \nonumber \\
        &+ \trace{\hat{P}_{t | t}^{-1}P_{t | t}} \label{eq:projection_objective}\\
        \subj &\chi_{3}^{2}(\gamma) \norm{A_{t, (i)}^{T}}_{P_{t | t}}^{2} + A_{t, (i)}\mu_{t | t} \leq b_{t, (i)}, \ \forall i. \label{eq:projection_constraint}
    \end{align}
\end{subequations}
With the objective function in \eqref{eq:projection_objective}, we seek to minimize the Kullback-Leibler (KL) divergence between the projected distribution ${\mcal{N}(\mu_{t | t}, P_{t | t})}$ and the distribution of the estimated pose computed from the UKF ${\mcal{N}(\hat{\mu}_{t | t}, \hat{P}_{t | t})}$. Furthermore, we introduce the constraint in \eqref{eq:projection_constraint} to enforce that the $\gamma\%$-confidence ellipsoid associated with the normal distribution $\mcal{N}(\mu_{t | t}, P_{t | t})$ lies within the safe polytope.

}

\subsection{Modifications to Observation Model}
\label{sec:observation_model}
Note that the pose estimation procedure in \eqref{eq:map_problem} requires the observations $y_{t}$ collected by the robot, which could be potentially high-dimensional. To further improve the computation time of the pose estimator, we design a procedure operating on the potentially high-dimensional observations (e.g., high-resolution images), transforming these observations into a more compact pseudo-measurement in $\SE(3)$. In addition, we assume that the measurement function $h$ extracts the $\mathfrak{se}(3)$ elements in its first input $\mu_{t | t - 1}^{(i)}$, transforming them to $\SE(3)$. Now, we present our procedure for computing the pseudo-measurement from the RGB image $I_{t}$.

Given an estimate of the robot's pose $\mu_{t | t - 1}$, we first render a set of 2D images of the GSplat map within a local neighborhood of $\mu_{t | t - 1}$. We then identify keypoints and the associated descriptors for points in these map images, as well as the robot's RGB image using the encoder-decoder-based method SuperPoint \cite{detone2018superpoint}. Next, we match features between both images using the transformer-based method LightGlue \cite{lindenberger2306lightglue}.\footnote{Other methods for feature detection and matching can be utilized, such as Scale-Invariant Feature Transform (SIFT) and Oriented FAST and Rotated BRIEF (ORB) \cite{karami2017image}.}
We again utilize RANSAC to remove outliers from the set of matched features to yield the final correspondence set ${\mcal{C}}$. Finally, we take the map points from $\mathcal{C}$ and project them into the 3D scene using the perspective projection model, the camera intrinsics, and the rendered depth from GSplat.

We can then estimate the pose of the robot using the Perspective-$n$-Point (P$n$P) problem, given by:
\begin{equation}
    \label{eq:pnp_problem}
    y_{t} = \argmin_{\mcal{P} \in \SE(3)} \sum_{(c_{w},c_{o}) \in \mcal{C}} \delta(\Pi(\mcal{P}, c_{w}), c_{o}),
\end{equation}
where ${y_{t} \in \SE(3)}$ denotes the pseudo-measurement describing the camera extrinsic parameters, $\mcal{C}$ denotes the set of correspondences between points in the environment $c_{w}$ and points in the observed RGB image $c_{o}$, ${\Pi: \mbb{R}^{3} \rightarrow \mbb{R}^{2}}$ represents the perspective projection model of point $p$ using the camera extrinsics $\mcal{P}$, and ${\delta: \mbb{R}^{2} \times \mbb{R}^{2} \rightarrow \mbb{R}}$ denotes the reprojection error. One common choice for $\delta$ is the $\ell_{2}$-squared-norm, which we utilize in this work. The $\ell_{2}$-squared-norm yields a non-linear least-squares optimization problem that can be solved using variants of the Levenberg-Marquardt and Gauss-Newton optimization algorithms. We note that various methods have been developed to solve the P$n$P problem such as the EP$n$P, DLS, and SQP$n$P methods \cite{terzakis2020consistently}.


Given $y_{t}$ and $\mu_{y, t}$, via ${\Phi: \SE(3) \times \mathfrak{se}(3) \mapsto \mathfrak{se}(3)}$, we first map $\mu_{y, t}$ to $\SE(3)$ and compute the measurement error in $\SE(3)$. We compute the rotation component of the error, given by ${R_{\ell} = R_{y(t)}^{T}R_{\mu_{y}(t)}}$ where $R_{y(t)}$ denotes the rotation described by $y_{t}$ and $R_{\mu_{y}(t)}$ denotes the rotation described by $\mu_{y, t}$. Likewise, we compute the translation error from the difference between the translation vectors described by $y_{t}$ and $\mu_{y, t}$. Lastly, we map the innovation in $SE(3)$ to $\mathfrak{se}(3)$ using the logarithmic map \cite{eade2013lie}. We note that the rotation component of the innovation in $\mathfrak{se}(3)$ corresponds to the angle-axis representation of the relative rotation between $R_{y(t)}$ and $R_{\mu_{y}(t)}$.

\begin{remark}[Lightweight Pose Estimator]
    \label{rem:pseudo_measurements}
    We note that the pseudo-measurement ${y_{t}}$, computed directly from the RGB image ${I_{t}}$ captured by the robot's onboard camera, is often of sufficiently-high accuracy. Hence, ${y_{t}}$ can be used directly as an estimate of the robot's pose in many navigation problems, especially in non-highly dynamic problems with feature-rich scenes. However, we note such an approach could be brittle in problems where the robot moves at high speeds in featureless regions. In the experiments, we assess the accuracy of this approach, comparing it to prior work in point-cloud registration and pose estimation in Gaussian Splatting environments.
    Even without the UKF, we can still obtain an estimate of the covariance associated with ${y_{t}}$ from the Hessian of \eqref{eq:pnp_problem} at its minimizer.
\end{remark}





\subsection{Global Initialization}
\label{sec:global_initialization}
The above pose estimation requires an estimate of the robot's pose (such as its neighborhood), which may not be available in many practical settings.
When a good initial guess of the robot's pose is unavailable, we execute a global pose estimation procedure. 
Our approach assumes that the robot has an onboard RGB-D camera with known camera intrinsics.
We first use the RGB-D image and camera intrinsics to generate a point cloud (in the camera frame). In addition, we generate a point cloud of the scene from the GSplat using the ellipsoid means $\{\mu_j\}_{j = 1}^{N}$, enabling the formulation of a point-cloud registration problem:
\begin{equation}
    \label{eq:point_to_point_registration}
    \mcal{P}^{*} = \argmin_{\mcal{P} \in \SE(3)} \sum_{(p, q) \in \mcal{C}} \norm{p - Tq}_{2}^{2},
\end{equation}
where the transformation matrix ${T \in \SE(3)}$ comprises of a translation component ${\rho \in \mbb{R}^{3}}$ and a rotation matrix ${R \in \SO(3)}$ and ${C}$ denotes the set of correspondences, associating the point $p$ in the map cloud to a point $q$ in the point cloud from the camera. Given a known set of correspondences, we can compute the optimal solution of \eqref{eq:point_to_point_registration} using Umeyama's method \cite{umeyama1991least}.

In practice, we do not have prior knowledge of the set of correspondences $\mcal{C}$ between the two point clouds. To address this challenge, we apply standard techniques in feature-based global point-cloud registration. We begin by computing $33$-dimensional Fast Point Feature Histograms (FPFH) descriptors \cite{rusu2009fast} for each point in the point-cloud, encoding the local geometric properties of each point. Prior work has shown that visual attributes can play an important role in improving the convergence speed of point-cloud registration algorithms \cite{park2017colored}, something that FPFH does not do. To solve this, we augment the FPFH feature descriptor of a given point with its RGB color. We then identify putative sets of correspondences using a nearest-neighbor query based on the augmented FPFH descriptors, before utilizing RANSAC to iteratively identify and remove outliers in $\mathcal{C}$. The RANSAC convergence criterion is based on the distance between the aligned point clouds and the length of a pair of edges defined by the set of correspondences.

}


\bibliographystyle{IEEEtran.bst}
\bibliography{citations.bib}

\end{document}